\documentclass{article} % For LaTeX2e
\usepackage[final]{colm2026_conference}

\usepackage{microtype}
\usepackage{hyperref}
\usepackage{url}
\usepackage{booktabs}

\usepackage{lineno}

\definecolor{darkblue}{rgb}{0, 0, 0.5}
\hypersetup{colorlinks=true, citecolor=darkblue, linkcolor=darkblue, urlcolor=darkblue}

\title{Model Merging via Data-Free Covariance Estimation}
\author{Marawan Gamal Abdel Hameed\thanks{
Corresponding author: marawan.gamal@mila.quebec
} \\
Mila \& DIRO, Université de Montréal \\
\And
Derek Tam \\
University of Toronto \& Vector Institute \\
\And
Pascal Jr Tikeng Notsawo \\
Mila \& DIRO, Université de Montréal \\
\And
Colin Raffel \\
University of Toronto \& Vector Institute \\
\And
Guillaume Rabusseau \\
Mila \& DIRO, Université de Montréal, CIFAR AI Chair \\
}

\newcommand{\cready}[1]{{#1}}

\usepackage{amsmath,amsfonts,bm}

\def\eqref#1{Equation~(\ref{#1})}
\def\1{\bm{1}}
\newcommand{\train}{\mathcal{D}}

\def\eps{{\epsilon}}

\def\vtheta{{\bm{\theta}}}

\def\vb{{\bm{b}}}

\def\vg{{\bm{g}}}

\def\vw{{\bm{w}}}
\def\vx{{\bm{x}}}
\def\vy{{\bm{y}}}
\def\vz{{\bm{z}}}

\def\mA{{\bm{A}}}
\def\mB{{\bm{B}}}
\def\mC{{\bm{C}}}

\def\mG{{\bm{G}}}

\def\mS{{\bm{S}}}

\def\mW{{\bm{W}}}

\def\mZ{{\bm{Z}}}

\def\mDelta{{\bm{\Delta}}}

\DeclareMathAlphabet{\mathsfit}{\encodingdefault}{\sfdefault}{m}{sl}
\SetMathAlphabet{\mathsfit}{bold}{\encodingdefault}{\sfdefault}{bx}{n}

\def\gS{{\mathcal{S}}}

\def\gX{{\mathcal{X}}}
\def\gY{{\mathcal{Y}}}

\newcommand{\E}{\mathbb{E}}
\newcommand{\Ls}{\mathcal{L}}
\newcommand{\R}{\mathbb{R}}

\DeclareMathOperator*{\argmin}{arg\,min}

\usepackage{amsthm, amssymb}
\usepackage{thmtools}
\usepackage{thm-restate}
\usepackage{enumitem}
\usepackage[ruled,vlined]{algorithm2e}
\usepackage{graphicx}
\usepackage{subcaption}
\usepackage{multirow}
\usepackage[table]{xcolor}
\usepackage{colortbl}
\usepackage{pifont}

\definecolor{lightgray}{gray}{0.9}
\definecolor{cteal}{HTML}{45AD8F}
\definecolor{xred}{HTML}{FF0000}

\newcommand{\cmark}{{\color{cteal}\ding{51}}}
\newcommand{\xmark}{{\color{xred}\ding{55}}}

\theoremstyle{plain}
\newtheorem{theorem}{Theorem}[section]

\newtheorem{lemma}[theorem]{Lemma}

\theoremstyle{definition}

\theoremstyle{remark}

\usepackage{tikz}
\usepackage{xspace}
\usepackage{nicematrix}
\usetikzlibrary{positioning, arrows.meta, decorations.pathreplacing}
\definecolor{ceigcov}{HTML}{059669}
\definecolor{ceigcovfill}{HTML}{D1FAE5}
\definecolor{ceigcovmed}{HTML}{34D399}
\definecolor{cregmean}{HTML}{C53030}
\definecolor{cregmeanfill}{HTML}{FED7D7}
\definecolor{cregmeanmed}{HTML}{F56565}
\definecolor{cgray}{HTML}{888780}
\definecolor{cgrayfill}{HTML}{F1EFE8}
\definecolor{cgraylight}{HTML}{B4B2A9}
\definecolor{cred}{HTML}{C03020}

\usepackage{titletoc}

\newcommand{\methodname}{ACTMat\xspace}
\newcommand{\figfont}{\footnotesize\sffamily}
\newcommand{\figtitlefont}{\small\sffamily\bfseries}

\DeclareMathOperator{\tr}{tr}

\DeclareMathSizes{10}{10}{6.7}{5}
\begin{document}

\ifcolmsubmission
\linenumbers
\fi

\maketitle

\begin{abstract}
Model merging provides a way of cheaply combining individual models to produce a model that inherits each individual's capabilities.
While some merging methods can approach the performance of multitask training, they are often heuristically motivated and lack theoretical justification. 
A principled alternative is to pose model merging as a layer-wise optimization problem that directly minimizes interference between tasks. 
However, this formulation requires estimating per-layer covariance matrices from data, which may not be available when performing merging. 
In contrast, many of the heuristically-motivated methods do not require auxiliary data, making them practically advantageous.
In this work, we revisit the interference minimization framework and show that, under certain conditions, covariance matrices can be estimated directly from \emph{difference matrices}, eliminating the need for data while also reducing computational costs. 
We validate our approach across vision and language benchmarks on models ranging from $86M$ parameters to $7B$ parameters, outperforming previous data-free state-of-the-art merging methods
\footnote{
Code available at \href{https://github.com/marawangamal/actmat}{github.com/marawangamal/actmat}
}. 
\end{abstract}

% \section{Introduction}
% % Good intros
% % Scalable Model Merging - https://openreview.net/pdf?id=xX8NJShgny
% % TIES - https://arxiv.org/pdf/2306.01708
% % CAT - https://arxiv.org/pdf/2505.06977
% % Dual anchors - https://arxiv.org/pdf/2510.21223
% % Backward Compatible representations - https://arxiv.org/pdf/2405.02581. This work discusses the application from the point of view of the fine-tuner which is nice and less confusing.

\section{Introduction}

Large-scale pretrained models have become the backbone of modern machine learning~\citep{bommasani2021opportunities}, and fine-tuning them for specific downstream tasks is now standard practice. This has led to a proliferation of publicly available task-specific expert models~\citep{wolf2020transformers}, each excelling in a narrow domain. 
However, many downstream applications demand capabilities that span multiple domains. Multitask learning~\citep{caruana1997mtl} and model ensembling are natural candidates for combining such capabilities, but the former requires simultaneous access to all training datasets, while the latter incurs significant storage and inference overhead at deployment.
In contrast, model merging combines expert capabilities by directly merging their parameters~\citep{utans1996weight, matena2022fisher, wortsman2022robust}. 
Though the ability to merge model parameters while retaining downstream performance is sometimes attributed to the \emph{linear mode connectivity} of checkpoints~\citep{frankle2020linear}, model merging still lacks a unifying theoretical foundation.
Consequently, a wide variety of merging methods with different underlying principles have been developed.

Despite their empirical success, state-of-the-art merging methods, such as TIES~\citep{prateek2023ties}, Iso-C~\citep{marczak2025notask} and TSV~\citep{gargiulo2025tsv} remain largely based off heuristics and lack theoretical guarantees. 
A notable exception is RegMean~\citep{jin2023dataless}, which frames model merging as a tractable layer-wise optimization objective.
Remarkably, although the RegMean objective considers layers independently (ignoring cross-layer interactions and nonlinearities),~\citet{sun2025cat} showed that it provides an upper bound on the difference in losses between the merged model and each expert model.
In addition to its desirable theoretical properties, we find that RegMean can actually outperform more recently proposed state-of-the-art methods when properly implemented and tuned. 
However, RegMean involves computing covariance matrices for each layer across all the tasks being merged, which requires access to each task's data when performing merging.
For most publicly available expert models, this training data is not released.
Even when data is available, computing and storing these matrices becomes prohibitively expensive for large-scale models. This limits RegMean's applicability in precisely the settings where model merging is most attractive.

This raises a natural question, \emph{can the covariance matrices required by RegMean be estimated without access to data?} 
In this work, we answer in the affirmative: under certain conditions, covariance matrices can be recovered directly from each task's \emph{difference matrix} (i.e., the difference between fine-tuned and pretrained matrices). We call this estimator ``\textbf{A}pproximating \textbf{C}ovariances via \textbf{T}ask Vectors for Activation \textbf{Mat}ching''~(\methodname). 
Combining \methodname with the RegMean objective, we obtain a fully data-free merging method that consistently outperforms prior state-of-the-art data-free approaches, as shown in Figures~\ref{fig:crown.jewel} and~\ref{fig:performance--merging-standard-benchmark}.

\begin{figure}[tb]
    \centering
    \begin{subfigure}[t]{0.55\linewidth}
        \vspace{0pt}
        \centering
        % regmean-vs-eigcov.tex
\begin{tikzpicture}[
    >=Stealth,
    scale=1.1, % Adjust this until it fits your 0.45\linewidth
    every node/.style={font=\figfont}
]

% === RegMean row ===
\node[anchor=west, font=\figtitlefont, text=cregmean] at (0, 4.0) {RegMean};
\node[anchor=west, font=\sffamily\scriptsize, text=cgray] at (0, 3.6) {requires data};

% Stacked model cards
\fill[cgrayfill, draw=cgray!40, rounded corners=2pt] (0.0, 2.35) rectangle ++(1.0, 0.7);
\fill[cgrayfill, draw=cgray!50, rounded corners=2pt] (0.07, 2.45) rectangle ++(1.0, 0.7);
\fill[cgrayfill, draw=cgray!60, rounded corners=2pt] (0.14, 2.55) rectangle ++(1.0, 0.7);
\node[font=\sffamily\scriptsize] at (0.64, 2.9) {$\mW_1\!\ldots\!\mW_T$};

% Plus sign between weights and data
\node[font=\sffamily\bfseries, text=cgray] at (1.5, 2.9) {$+$};

% Data cylinder icon
\begin{scope}[shift={(1.85, 2.55)}]
    \fill[cregmeanfill] (0, 0.1) -- (0, 0.55) arc[start angle=180, end angle=360, x radius=0.35, y radius=0.12] -- (0.7, 0.1) arc[start angle=0, end angle=-180, x radius=0.35, y radius=0.12];
    \draw[cregmean!50] (0, 0.1) -- (0, 0.55);
    \draw[cregmean!50] (0.7, 0.1) -- (0.7, 0.55);
    \draw[cregmean!50] (0, 0.55) arc[start angle=180, end angle=360, x radius=0.35, y radius=0.12];
    \fill[cregmeanfill] (0, 0.55) arc[start angle=180, end angle=0, x radius=0.35, y radius=0.12] arc[start angle=0, end angle=-360, x radius=0.35, y radius=0.12];
    \draw[cregmean!50] (0.35, 0.55) ellipse[x radius=0.35, y radius=0.12];
    \draw[cregmean!50] (0, 0.1) arc[start angle=180, end angle=360, x radius=0.35, y radius=0.12];
    \draw[cregmean!30] (0.12, 0.38) -- (0.58, 0.38);
    \draw[cregmean!30] (0.12, 0.25) -- (0.58, 0.25);
\end{scope}
\node[font=\sffamily\tiny, text=cgray] at (2.2, 2.35) {data};

% Arrow to heatmap
\draw[->, cregmean!60, thick] (2.7, 2.9) -- (3.2, 2.9);

% Red heatmap 5x5 for C_t
\begin{scope}[shift={(3.25, 2.3)}]
    \clip[rounded corners=3pt] (-0.04, -0.04) rectangle (1.24, 1.24);
    \fill[white] (-0.04, -0.04) rectangle (1.24, 1.24);
    \fill[cregmean, opacity=0.92] (0.00, 0.96) rectangle ++(0.24, 0.24);
    \fill[cregmean, opacity=0.50] (0.24, 0.96) rectangle ++(0.24, 0.24);
    \fill[cregmean, opacity=0.15] (0.48, 0.96) rectangle ++(0.24, 0.24);
    \fill[cregmean, opacity=0.25] (0.72, 0.96) rectangle ++(0.24, 0.24);
    \fill[cregmean, opacity=0.08] (0.96, 0.96) rectangle ++(0.24, 0.24);
    \fill[cregmean, opacity=0.50] (0.00, 0.72) rectangle ++(0.24, 0.24);
    \fill[cregmean, opacity=0.88] (0.24, 0.72) rectangle ++(0.24, 0.24);
    \fill[cregmean, opacity=0.40] (0.48, 0.72) rectangle ++(0.24, 0.24);
    \fill[cregmean, opacity=0.12] (0.72, 0.72) rectangle ++(0.24, 0.24);
    \fill[cregmean, opacity=0.20] (0.96, 0.72) rectangle ++(0.24, 0.24);
    \fill[cregmean, opacity=0.15] (0.00, 0.48) rectangle ++(0.24, 0.24);
    \fill[cregmean, opacity=0.40] (0.24, 0.48) rectangle ++(0.24, 0.24);
    \fill[cregmean, opacity=0.82] (0.48, 0.48) rectangle ++(0.24, 0.24);
    \fill[cregmean, opacity=0.48] (0.72, 0.48) rectangle ++(0.24, 0.24);
    \fill[cregmean, opacity=0.18] (0.96, 0.48) rectangle ++(0.24, 0.24);
    \fill[cregmean, opacity=0.25] (0.00, 0.24) rectangle ++(0.24, 0.24);
    \fill[cregmean, opacity=0.12] (0.24, 0.24) rectangle ++(0.24, 0.24);
    \fill[cregmean, opacity=0.48] (0.48, 0.24) rectangle ++(0.24, 0.24);
    \fill[cregmean, opacity=0.78] (0.72, 0.24) rectangle ++(0.24, 0.24);
    \fill[cregmean, opacity=0.38] (0.96, 0.24) rectangle ++(0.24, 0.24);
    \fill[cregmean, opacity=0.08] (0.00, 0.00) rectangle ++(0.24, 0.24);
    \fill[cregmean, opacity=0.20] (0.24, 0.00) rectangle ++(0.24, 0.24);
    \fill[cregmean, opacity=0.18] (0.48, 0.00) rectangle ++(0.24, 0.24);
    \fill[cregmean, opacity=0.38] (0.72, 0.00) rectangle ++(0.24, 0.24);
    \fill[cregmean, opacity=0.72] (0.96, 0.00) rectangle ++(0.24, 0.24);
    \draw[cregmean!35, rounded corners=3pt] (-0.04, -0.04) rectangle (1.24, 1.24);
\end{scope}
\node[font=\sffamily\scriptsize, text=cregmean] at (3.85, 2.05) 
{$\mC_t = \E[\vz\vz^\top]$};

% Arrow to W*
\draw[->, cgray!60, thick] (4.6, 2.9) -- (5.1, 2.9);

% Merged W*
\fill[cgrayfill, draw=cgray!50, rounded corners=4pt] (5.15, 2.5) rectangle ++(1.0, 0.8);
\node[font=\sffamily\bfseries] at (5.65, 2.9) {$W^*$};

% === \methodname row ===
\node[anchor=west, font=\figtitlefont, text=ceigcov] at (0, 1.5) {\methodname};
\node[anchor=west, font=\sffamily\scriptsize, text=cgray] at (0, 1.1) {data-free (ours)};

% Stacked model cards
\fill[cgrayfill, draw=cgray!40, rounded corners=2pt] (0.0, -0.15) rectangle ++(1.0, 0.7);
\fill[cgrayfill, draw=cgray!50, rounded corners=2pt] (0.07, -0.05) rectangle ++(1.0, 0.7);
\fill[cgrayfill, draw=cgray!60, rounded corners=2pt] (0.14, 0.05) rectangle ++(1.0, 0.7);
\node[font=\sffamily\scriptsize] at (0.64, 0.4) {$\mW_1\!\ldots\!\mW_T$};

% Plus sign (faded to match ghosted cylinder)
\node[font=\sffamily\bfseries, text=cgray!30] at (1.5, 0.4) {$+$};

% Ghosted data cylinder — same position as RegMean's but faded
\begin{scope}[shift={(1.85, 0.05)}, opacity=0.25]
    \fill[cgrayfill] (0, 0.1) -- (0, 0.55) arc[start angle=180, end angle=360, x radius=0.35, y radius=0.12] -- (0.7, 0.1) arc[start angle=0, end angle=-180, x radius=0.35, y radius=0.12];
    \draw[cgray!50] (0, 0.1) -- (0, 0.55);
    \draw[cgray!50] (0.7, 0.1) -- (0.7, 0.55);
    \draw[cgray!50] (0, 0.55) arc[start angle=180, end angle=360, x radius=0.35, y radius=0.12];
    \fill[cgrayfill] (0, 0.55) arc[start angle=180, end angle=0, x radius=0.35, y radius=0.12] arc[start angle=0, end angle=-360, x radius=0.35, y radius=0.12];
    \draw[cgray!50] (0.35, 0.55) ellipse[x radius=0.35, y radius=0.12];
    \draw[cgray!50] (0, 0.1) arc[start angle=180, end angle=360, x radius=0.35, y radius=0.12];
    \draw[cgray!30] (0.12, 0.38) -- (0.58, 0.38);
    \draw[cgray!30] (0.12, 0.25) -- (0.58, 0.25);
\end{scope}
\node[font=\sffamily\tiny, text=cgray!40] at (2.2, -0.15) {data};
% Red X over ghosted cylinder
\draw[cred, line width=1.5pt, line cap=round] (1.85, 0.08) -- (2.55, 0.72);
\draw[cred, line width=1.5pt, line cap=round] (2.55, 0.08) -- (1.85, 0.72);

% Arrow from weights to heatmap
\draw[->, ceigcov!70, thick] (2.7, 0.4) -- (3.2, 0.4);

% Green heatmap 5x5 for Delta^T Delta
\begin{scope}[shift={(3.25, -0.2)}]
    \clip[rounded corners=3pt] (-0.04, -0.04) rectangle (1.24, 1.24);
    \fill[white] (-0.04, -0.04) rectangle (1.24, 1.24);
    \fill[ceigcov, opacity=0.88] (0.00, 0.96) rectangle ++(0.24, 0.24);
    \fill[ceigcov, opacity=0.45] (0.24, 0.96) rectangle ++(0.24, 0.24);
    \fill[ceigcov, opacity=0.18] (0.48, 0.96) rectangle ++(0.24, 0.24);
    \fill[ceigcov, opacity=0.28] (0.72, 0.96) rectangle ++(0.24, 0.24);
    \fill[ceigcov, opacity=0.10] (0.96, 0.96) rectangle ++(0.24, 0.24);
    \fill[ceigcov, opacity=0.45] (0.00, 0.72) rectangle ++(0.24, 0.24);
    \fill[ceigcov, opacity=0.82] (0.24, 0.72) rectangle ++(0.24, 0.24);
    \fill[ceigcov, opacity=0.35] (0.48, 0.72) rectangle ++(0.24, 0.24);
    \fill[ceigcov, opacity=0.15] (0.72, 0.72) rectangle ++(0.24, 0.24);
    \fill[ceigcov, opacity=0.22] (0.96, 0.72) rectangle ++(0.24, 0.24);
    \fill[ceigcov, opacity=0.18] (0.00, 0.48) rectangle ++(0.24, 0.24);
    \fill[ceigcov, opacity=0.35] (0.24, 0.48) rectangle ++(0.24, 0.24);
    \fill[ceigcov, opacity=0.78] (0.48, 0.48) rectangle ++(0.24, 0.24);
    \fill[ceigcov, opacity=0.42] (0.72, 0.48) rectangle ++(0.24, 0.24);
    \fill[ceigcov, opacity=0.20] (0.96, 0.48) rectangle ++(0.24, 0.24);
    \fill[ceigcov, opacity=0.28] (0.00, 0.24) rectangle ++(0.24, 0.24);
    \fill[ceigcov, opacity=0.15] (0.24, 0.24) rectangle ++(0.24, 0.24);
    \fill[ceigcov, opacity=0.42] (0.48, 0.24) rectangle ++(0.24, 0.24);
    \fill[ceigcov, opacity=0.72] (0.72, 0.24) rectangle ++(0.24, 0.24);
    \fill[ceigcov, opacity=0.35] (0.96, 0.24) rectangle ++(0.24, 0.24);
    \fill[ceigcov, opacity=0.10] (0.00, 0.00) rectangle ++(0.24, 0.24);
    \fill[ceigcov, opacity=0.22] (0.24, 0.00) rectangle ++(0.24, 0.24);
    \fill[ceigcov, opacity=0.20] (0.48, 0.00) rectangle ++(0.24, 0.24);
    \fill[ceigcov, opacity=0.35] (0.72, 0.00) rectangle ++(0.24, 0.24);
    \fill[ceigcov, opacity=0.68] (0.96, 0.00) rectangle ++(0.24, 0.24);
    \draw[ceigcov!35, rounded corners=3pt] (-0.04, -0.04) rectangle (1.24, 1.24);
\end{scope}
\node[font=\sffamily\scriptsize, text=ceigcov] at (3.85, -0.42) {$\mC_t \approx \mDelta_t^\top\!\mDelta_t$};

% Arrow to W*
\draw[->, ceigcov!70, thick] (4.6, 0.4) -- (5.1, 0.4);

% Merged W*
\fill[cgrayfill, draw=cgray!50, rounded corners=4pt] (5.15, 0.0) rectangle ++(1.0, 0.8);
\node[font=\sffamily\bfseries] at (5.65, 0.4) {$W^*$};

\end{tikzpicture}
    \end{subfigure}%
    \hfill%
    \begin{subfigure}[t]{0.42\linewidth}
        \vspace{0pt}
        \centering
        % barchart-t5large.tex
% Requires: \usepackage{tikz}, \usetikzlibrary{arrows.meta}
% Requires: \input{eigcov-colors.tex} in preamble

\begin{tikzpicture}[
    >=Stealth,
    scale=1.1, % Adjust this until it fits your 0.45\linewidth
    every node/.style={font=\figfont}
]

\node[anchor=west, font=\figtitlefont] at (0, 4.0) {T5-Large (7 NLP tasks)};

\def\barorigin{1.8}
\def\barscale{0.09}
\def\barh{0.32}
\def\barsep{0.55}

% Average: 50.9 (light coral)
\pgfmathsetmacro{\yA}{3.2}
\pgfmathsetmacro{\wA}{(50.9-48)*\barscale}
\node[anchor=east, font=\sffamily\scriptsize, text=cgray] at (\barorigin-0.12, \yA) {Average};
\fill[ceigcovmed, opacity=0.2, rounded corners=2pt]
    (\barorigin, \yA-\barh/2) rectangle ++(\wA, \barh);
\node[anchor=west, font=\sffamily\scriptsize, text=cgray] at ({\barorigin+\wA+0.1}, \yA) {50.9};

% ISO-C: 57.7 (light coral)
\pgfmathsetmacro{\yB}{\yA-\barsep}
\pgfmathsetmacro{\wB}{(57.7-48)*\barscale}
\node[anchor=east, font=\sffamily\scriptsize, text=cgray] at (\barorigin-0.12, \yB) {ISO-C};
\fill[ceigcovmed, opacity=0.2, rounded corners=2pt]
    (\barorigin, \yB-\barh/2) rectangle ++(\wB, \barh);
\node[anchor=west, font=\sffamily\scriptsize, text=cgray] at ({\barorigin+\wB+0.1}, \yB) {57.7};

% TSV: 74.5 (light coral)
\pgfmathsetmacro{\yC}{\yA-2*\barsep}
\pgfmathsetmacro{\wC}{(74.5-48)*\barscale}
\node[anchor=east, font=\sffamily\scriptsize, text=cgray] at (\barorigin-0.12, \yC) {TSV};
\fill[ceigcovmed, opacity=0.2, rounded corners=2pt]
    (\barorigin, \yC-\barh/2) rectangle ++(\wC, \barh);
\node[anchor=west, font=\sffamily\scriptsize, text=cgray] at ({\barorigin+\wC+0.1}, \yC) {74.5};

% EigenCov: 79.8 (coral, solid, slightly taller)
\pgfmathsetmacro{\yD}{\yA-3*\barsep}
\pgfmathsetmacro{\wD}{(79.8-48)*\barscale}
\def\barhbig{0.40}
\node[anchor=east, font=\sffamily\bfseries\scriptsize, text=ceigcov] at (\barorigin-0.12, \yD) {\methodname (Ours)};
\fill[ceigcovmed, opacity=0.7, rounded corners=2pt]
    (\barorigin, \yD-\barhbig/2) rectangle ++(\wD, \barhbig);
\node[anchor=west, font=\sffamily\bfseries\scriptsize, text=ceigcov] at ({\barorigin+\wD+0.1}, \yD) {79.8};

% RegMean: 80.8 (solid teal, saturated)
\pgfmathsetmacro{\yE}{\yA-4*\barsep}
\pgfmathsetmacro{\wE}{(80.8-48)*\barscale}
\node[anchor=east, font=\sffamily\scriptsize, text=cregmean] at (\barorigin-0.12, \yE) {RegMean};
\fill[cregmeanmed, opacity=0.75, rounded corners=2pt]
    (\barorigin, \yE-\barh/2) rectangle ++(\wE, \barh);
\node[anchor=west, font=\sffamily\scriptsize, text=cregmean] at ({\barorigin+\wE+0.1}, \yE) {80.8};

% Legend
\pgfmathsetmacro{\yL}{\yE - 0.7}
\fill[cregmeanmed, opacity=0.75, rounded corners=1pt] (1.2, \yL-0.1) rectangle ++(0.25, 0.2);
\node[anchor=west, font=\sffamily\tiny, text=cgray] at (1.55, \yL) {uses data};

\fill[ceigcovmed, opacity=0.7, rounded corners=1pt] (3.0, \yL-0.1) rectangle ++(0.25, 0.2);
\node[anchor=west, font=\sffamily\tiny, text=cgray] at (3.35, \yL) {data-free};

\end{tikzpicture}
    \end{subfigure}
    \caption{\textbf{Left:} RegMean requires data to compute activation covariances $\mC_t$, while \methodname estimates them directly from the difference matrices as $\mDelta_t^\top\mDelta_t \approx \mathbf{C}_t$. \textbf{Right:} On T5-Large, \methodname nearly matches RegMean's accuracy without any data, while substantially outperforming other data-free baselines.}
    \label{fig:crown.jewel}
\end{figure}

\section{Related Work}

\textbf{Mode Connectivity.}
\cite{draxler2018essentially} found that minima in independently trained models are connected by nonlinear paths along which the loss does not increase.
\cite{frankle2020linear} showed that as the models being interpolated share an increasingly longer training trajectory, a linear path with no loss barrier can be found. Similarly, \cite{mcmahan17communication} reported linear mode connectivity between models sharing the same initial random seed.
Furthermore, \cite{ainsworth2023git} showed that a linear path can be found even when  models do not share a training trajectory, by accounting for the permutation symmetries prevalent in neural networks~\citep{entezari2022permutation}.
While the aforementioned works focus on models trained on the same data, \cite{gueta2023knowledge} studied the more prevalent setting where models are trained on different datasets.

\textbf{Model Merging.}
Many merging methods rely on the models being merged having shared a training trajectory, thus avoiding the need to account for permutation symmetries~\citep{wortsman2022robust, matena2022fisher, jin2023dataless}.
\cite{ilharco2023editing} interpret model merging as the weighted addition of \emph{task vectors}, defined as the parameter-space difference between a fine-tuned model and its pretrained initialization.
Reinterpreting model merging within this framework has fostered theoretical analysis highlighting \emph{weight disentanglement} as a necessary condition for effective task arithmetic~\citep{ortizjimenez2023tangent}. Interestingly, fine-tuning in the linear regime has been shown to promote weight disentanglement~\citep{ortizjimenez2023tangent, yoshida2024mastering}. 

\textbf{Merging Methods.}
Beyond  linear combinations of parameters~\citep{wortsman2022robust, ilharco2023editing}, a number of methods have been proposed based on the principle of \emph{interference minimization}. TIES~\citep{prateek2023ties} reduces interference at the parameter level, by trimming low-magnitude parameters, resolving sign conflicts across task vectors, and merging only sign-consistent updates. 
Similarly, DARE~\citep{yu2024language} resets a fraction of fine-tuned parameters to their original weights at random, effectively reducing parameter interference by merging sparse task vectors.
Another family of methods leverages the matrix structure of linear layers via the Singular Value Decomposition (SVD). 
In Task Singular Vectors (TSV)~\citep{gargiulo2025tsv}, the authors observe that task matrices are inherently low-rank and reduce ``Singular Task Interference'' by decorrelating the singular vectors of different tasks before merging.
Iso-C~\citep{marczak2025notask} flattens the spectrum of the merged matrix to balance out dominant directions in weight space with underrepresented ones. KnOTS~\citep{stoica2025knots} finds that models fine-tuned with LoRA exhibit a significantly lower centered kernel alignment score~\citep{kornblith2019cka}, compared with full fine-tuning and propose to merge models in an aligned space via the SVD.
Notably, methods such as TSV and Iso-C offer data-free settings, as they have been shown to be relatively robust to scaling coefficients. 
Meanwhile, methods such as RegMean~\citep{jin2023dataless}, LOT~\citep{sun2025lot}, WUDI~\citep{cheng2025whoever}, and AdaMerging~\citep{yang2024adamerging} require the use of auxiliary data either for optimization of data-dependent objectives or for hyper-parameter tuning.
Similarly to TSV and Iso-C, \methodname is entirely data-free.

\textbf{Kronecker-Factored Approximate Curvature (KFAC).} KFAC~\citep{martens2015optimizing} provides a tractable approximation of the Fisher information matrix by assuming (i) that the Fisher matrix is block-diagonal across layers and (ii) that layer-wise activations and activation gradients are uncorrelated. In our work, we leverage a variation on the latter assumption to derive the \methodname  estimator.

\section{Method}

% ====================================
% Overview of Method section:
% ====================================
% 0. Notation & Prelim
% 1. Introduce the interference objective 
% 2. Following LOT merging's derivation, show that layer-wise interference can be used to upper bound global interference (i.e., \Delta\Ls_k = \beta\sum\prod\gamma^{(l)} \|\Delta f_k\|)
% 3. Present the RegMean minimizer, showing the need to compute layer-wise covariances
% 4. Present our main theorem on covariance approximation
% 5. Reference to appendix here somewhere to show how we handle the sequence dimension

We begin by introducing the model merging setting under consideration and
the interference minimization framework together with the \methodname merging rule (Section~\ref{sec:method--interference-minimization}).
In Section~\ref{sec:method--covariance-estimation}, we formally establish an upper bound on the approximation error of the \methodname covariance estimator and provide empirical evidence that this bound is tight. 
We conclude with an analysis of the \methodname estimator's behavior within the interference minimization framework (Sections~\ref{sec:method--scaling-coefficients}~\&~\ref{sec:method--regmean-guarantees}).

\paragraph{Background and Notation.}\label{sec:method--background}

We consider the standard model merging setting in which $T$ models with the same architecture and pretrained initialization are fine-tuned on different tasks, then combined into a single model. Specifically, let $f: \gX \times \Theta \to \gY$ be a neural network parameterized by $\vtheta \in \Theta$.
Each task $t$ is associated with a discrete distribution $\train_t$ over $\gX$. Fine-tuning on task $t$ involves updating the model parameters, starting at initial parameters $\vtheta_0$ and ending with task-specific parameters $\vtheta_t$. 
For an arbitrary linear layer in $f$, we denote its pretrained parameters by $\mW_0 \in \R^{D_o \times D_i}$, its fine-tuned parameters for task $t$ by $\mW_t \in \R^{D_o \times D_i}$, its input by $\vz \in \R^{D_i}$, its output by $\vy = \mW\vz \in \R^{D_o}$, and its difference matrix by $\mDelta_t := \mW_t - \mW_0$. 
% V3
We use $\vx$ to denote the inputs to the model and $\vx\sim\train_t$ to indicate that they are sampled according to the $t$-th distribution.
With some abuse of notation we also use $\vz\sim\train_t$ to indicate the induced distribution over an arbitrary linear layer's inputs, when model inputs $\vx$ are sampled from the $t$-th distribution.

\subsection{Model Merging as Interference Minimization}\label{sec:method--interference-minimization}

Following \cite{jin2023dataless}, we formulate model merging as a layer-wise optimization problem. For each linear layer, we seek the merged weights that best preserve each task's activations:
\begin{equation}
    \label{eq:interference}
    \mW^\star \in \argmin_{\mW} \sum_{t=1}^T \E_{\vz\sim\train_t}\left[
    \left\| \mW\vz - \mW_t \vz \right\|_2^2\right]
\end{equation}
This objective is solved independently per layer and admits
\begin{equation}
    \label{eq:interference-minimizer}
   \mW^\star = \sum_{t=1}^T \mW_t \mC_t \Big(\sum_{t^\prime} \mC_{t^\prime}\Big)^\dagger, 
\end{equation}
as the minimum Frobenius norm solution (see Lemma~\ref{lem:interference-minimization}),
where $\mC_t = \E_{\vz\sim\train_t}[\vz\vz^\top]$ denotes the second moment of the layer inputs under distribution $\train_t$ and $^\dagger$ the Moore–Penrose pseudoinverse. 
Under this formulation, model merging reduces to a covariance estimation problem%
\footnote{Strictly speaking, $\mC_t$ is a second moment matrix rather than a centered covariance, though we refer to $\mC_t$ as a covariance matrix throughout for brevity.}.
In the following section, we show that, under certain conditions, the covariance matrix can be approximated from the difference matrices as $\mC_t \approx \mDelta_t^\top\mDelta_t$, yielding a \textbf{fully data-free merging rule}:
\begin{equation}
    \label{eq:eigencov}
    \text{({\texttt{\methodname}} \textit{Estimator})}\hspace{1cm}
    \boxed{
    \mW^\star \approx \sum_{t=1}^T \mW_t 
    \Big(\mDelta_t^\top\mDelta_t\Big) 
    \Big(\sum_{t^\prime=1}^T \mDelta_{t^\prime}^\top\mDelta_{t^\prime}\Big)^\dagger,
    }
    \phantom{
    \text{\textbf{\texttt{\methodname}} Estimator:}\hspace{1cm}
    }
\end{equation}
which we refer to as 
``\textbf{A}pproximating \textbf{C}ovariances via \textbf{T}ask Vectors for Activation \textbf{Mat}ching''~(\methodname).

\subsection{Covariance Estimation of Activations}\label{sec:method--covariance-estimation}

In this section, we show that the covariance matrices of activations in~\eqref{eq:interference-minimizer} can be approximated directly from the difference matrices, up to a scaling factor. In other words, we show that the \emph{angular distance} between $\mDelta_t^\top\mDelta_t$ and $\mC_t$ is small, where
$
\measuredangle(\mA,\mB) := \arccos\!\left(
 \langle \mA, \mB \rangle_F / (\|\mA\|_F \, \|\mB\|_F)
\right)
$
denotes the angular distance metric%
\footnote{
$\langle \mA, \mB \rangle_F := \mathrm{tr}(\mA^\top \mB)$ denotes the Frobenius inner product and $\| \mA \|_F := \sqrt{ \langle \mA, \mA \rangle_F }$ the associated norm.
}. Consider a linear layer fine-tuned using full-batch gradient descent for $K$ iterations with a fixed learning rate $\eta$. Let $\vz^{(k)}$ denote the layer's input at iteration $k$, $\vy^{(k)} = \mW^{(k)}\vz^{(k)}$ its output, and $\vg^{(k)} := \nabla_\vy \Ls(\vx;\vtheta^{(k)})$ the gradient of the loss with respect to the output. Using the chain rule, the gradient with respect to $\mW$ at iteration $k$ is
\[
\E_{\vx\sim\train_t}\!\left[\nabla_\mW \Ls(\vx;\vtheta^{(k)})\right] = \E_{\vx\sim\train_t}\!\left[\vg^{(k)}  \vz^{(k)\top}\right].
\]
Since $\mDelta_t = -\eta\sum_{k=0}^{K}\E_{\vx\sim\train_t}[\vg^{(k)}  \vz^{(k)\top}]$, one can easily check that 
\begin{equation}\label{eq:dtd-expansion}
\mDelta_t^\top\mDelta_t \propto 
\textstyle\sum_{k,k'} \E_{\vx,\vx'\sim\train_t}\left[\vz^{(k)}\vz^{'(k')\top}
\,
\vg^{(k)\top} \vg^{'(k')}\right],
\end{equation}
suggesting that the product $\mDelta_t^\top\mDelta_t$ \emph{captures} second-order statistics of the layer's inputs.
However, recovering the covariance of activations at the end of training, $\textstyle\mC_t^{(K)} = \E[\vz^{(K)}\vz^{(K)\top}]$, from this expression is not immediate. 
In the following theorem, we show that the angular distance between $\mDelta_t^\top\mDelta_t$ and $\mC_t^{(K)}$ is upper bounded by three error terms, under a simplified training regime (proof in Appendix~\ref{app:theorems--covariance-estimation}).
In practice, we find that each of these three error terms is relatively small, indicating that $\mDelta_t^\top\mDelta_t$ is approximately proportional to $\mC_t^{(K)}$.

\begin{restatable}[Covariance Estimation]{theorem}{covarianceEstimationThm}
\label{thm:cov-estimation}
Consider a linear layer fine-tuned using full-batch gradient descent for $K$ iterations with learning rate $\eta$, and let $\vz^{(k)}$, $\vg^{(k)}$ denote the layer's input and its output gradient at iteration $k$, respectively. Define the accumulated gradient mean, accumulated second moment, and accumulated uncorrelated second moment as
\begin{align*}
\overline{\mG} := \sum_{k=0}^{K} 
\E
\left[\vg^{(k)}\vz^{(k)\top}\right], \,
\overline{\mS} := \sum_{k=0}^{K} 
\E
\left[
\vz^{(k)} 
\vz^{(k)\top}
\|\vg^{(k)}\|^2
\right],  \,
\widetilde{\mS} := \sum_{k=0}^{K} 
\E
\left[\vz^{(k)}\vz^{(k)\top}\right] 
\E
\left[\|\vg^{(k)}\|^2\right],
\end{align*}
where the expectation is taken over the $t$-th distribution $\train_t$,
and $\|\cdot\|$ denotes the Euclidean norm.
Then, the angular distance between $\mDelta_t^\top\mDelta_t$ and the final covariance $\mC_t^{(K)}$ satisfies
\begin{equation}
\label{eq:thm-angular-distance-bound}
\measuredangle\big(\mDelta_t^\top \mDelta_t,\; \mC_t^{(K)}\big) \leq \eps^{(\mathrm{cross})} + \eps^{(\mathrm{corr})} + \eps^{(\mathrm{drift})},
\end{equation}
where $\eps^{(\mathrm{cross})} = \measuredangle\big(\overline{\mG}^\top\overline{\mG},\; \overline{\mS}\big)$ is the cross-term error, $\eps^{(\mathrm{corr})} = \measuredangle\big(\overline{\mS},\; \widetilde{\mS}\big)$ is the correlation error, and $\eps^{(\mathrm{drift})} = \measuredangle\big(\widetilde{\mS},\; \mC_t^{(K)}\big)$ is the drift error. In particular, $\mDelta_t^\top \mDelta_t \propto \mC_t^{(K)}$ when all three errors vanish.
\end{restatable}
\cready{
\begin{proof}[Proof sketch] After K iterations of gradient descent with a fixed learning, the difference matrix is given by $\mDelta_t = - \eta \overline{\mG}$. Successively applying the triangle inequality leads to 
$
\measuredangle\!\big(\mDelta_t^\top\mDelta_t,\; \mC_t^{(K)}\big) 
\;\leq\; 
\measuredangle\!\big(\mDelta_t^\top\mDelta_t,\; \overline{\mG}^\top\overline{\mG}\big) 
\;+\; \measuredangle\!\big(\overline{\mG}^\top\overline{\mG},\; \overline{\mS}\big) 
\;+\; \measuredangle\!\big(\overline{\mS},\; \widetilde{\mS}\big) 
\;+\; \measuredangle\!\big(\widetilde{\mS},\; \mC_t^{(K)}\big)
$. Since $\mDelta_t$ and $\overline{\mG}$ are collinear, the first term vanishes, leading to~\eqref{eq:thm-angular-distance-bound}.
\end{proof}
}

In order to analyze the contributions of each of the error terms in Theorem~\ref{thm:cov-estimation}, we
fine-tune the ViT-B/16~\citep{dosovitskiy2021vit} model on eight downstream tasks (Cars~\citep{krause2013cars}, DTD~\citep{cimpoi2014dtd}, 
EuroSAT~\citep{helber2019eurosat}, GTSRB~\citep{stallkamp2011gtsrb}, MNIST~\citep{lecun1998mnist}, RESISC45~\citep{cheng2017resisc45}, SUN397~\citep{xiao2016sun397}, and SVHN~\citep{netzer2011svhn}).

\begin{figure}
    \centering
    \begin{subfigure}[t]{0.31\linewidth}
        \centering
        \includegraphics[width=\linewidth]{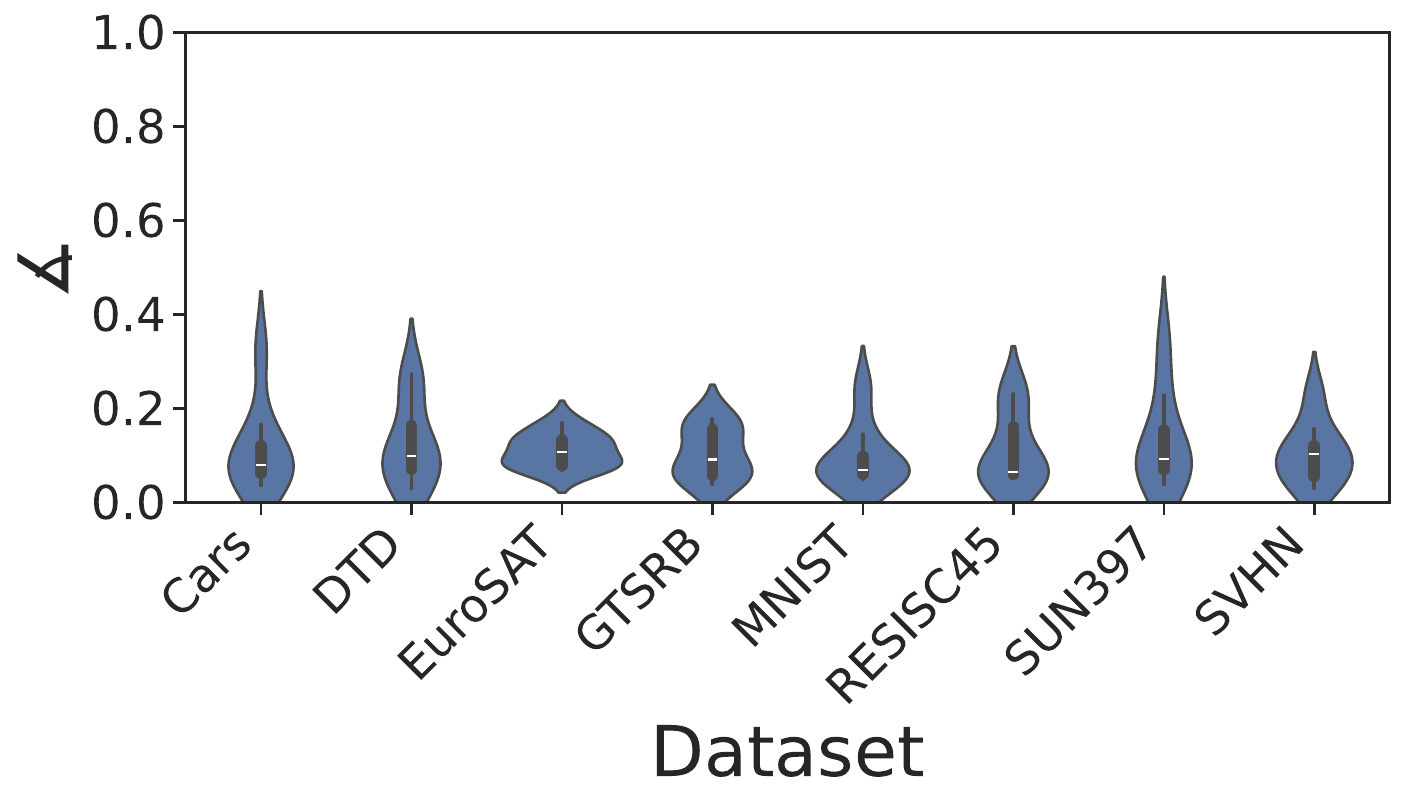}
        \caption{Cross-term error}
        \label{fig:ang-dist-cross}
    \end{subfigure}
    \hfill
    \begin{subfigure}[t]{0.31\linewidth}
        \centering
        \includegraphics[width=\linewidth]{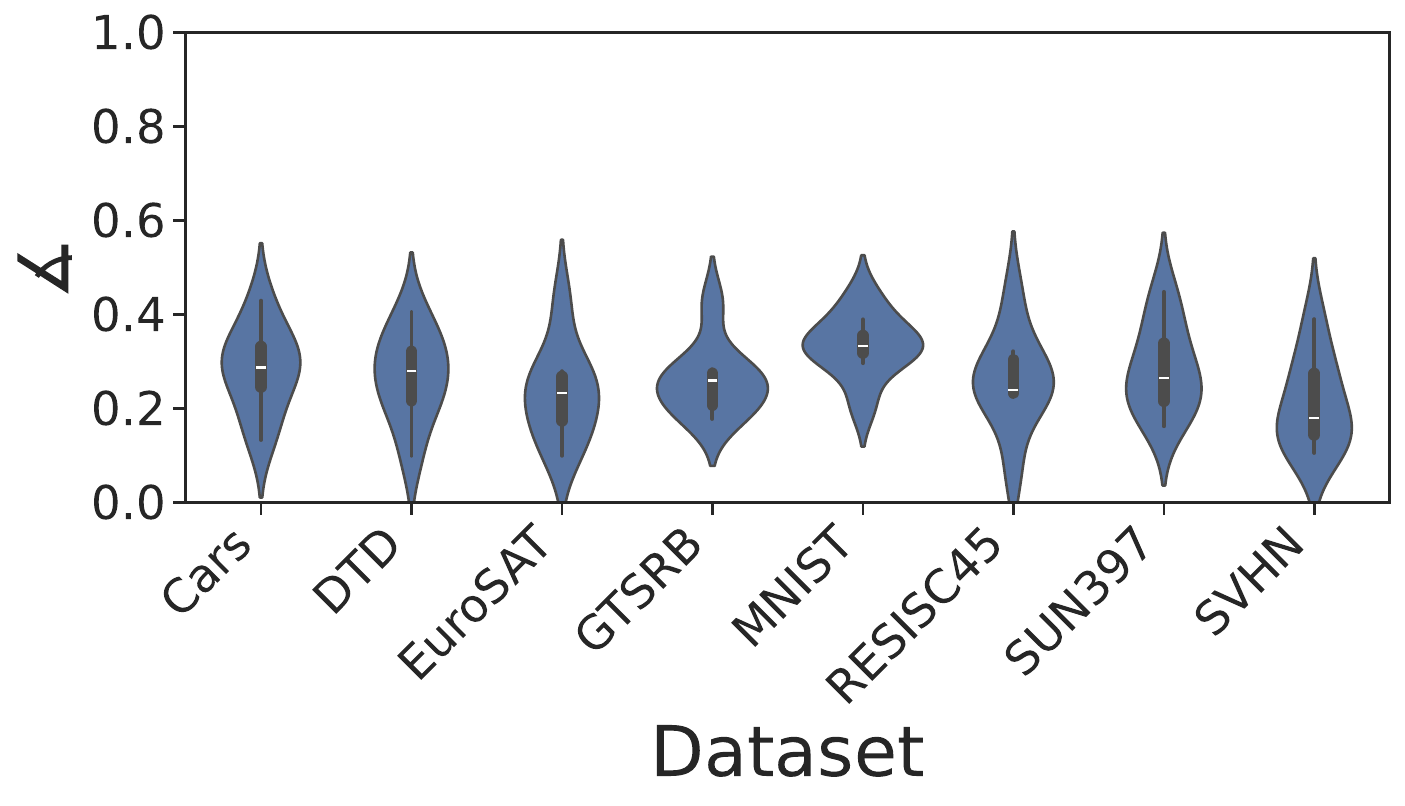}
        \caption{Correlation error}
        \label{fig:ang-dist-corr}
    \end{subfigure}
    \hfill
    % \begin{subfigure}[t]{0.34\linewidth}
    %     \centering
    %     \includegraphics[width=\linewidth]{figures/angular_distance_drift_trajectory.pdf}
    %     \caption{Drift error}
    %     \label{fig:ang-dist-drift}
    % \end{subfigure}
    \begin{subfigure}[t]{0.31\linewidth}
        \centering
        \includegraphics[width=\linewidth]{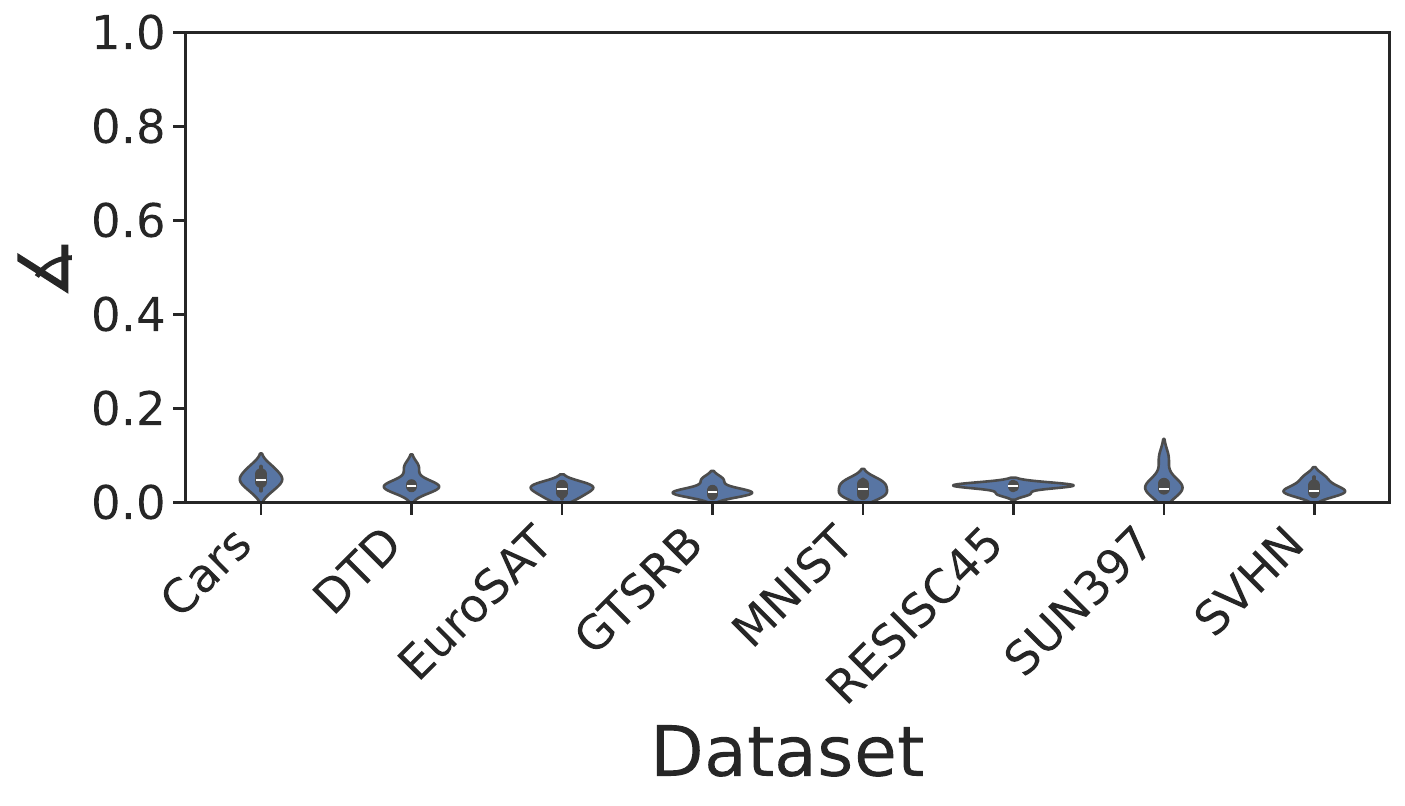}
        \caption{Drift error}
        \label{fig:ang-dist-drift}
    \end{subfigure}
    \caption{Empirical measurement of the three angular error terms in Theorem~\ref{thm:cov-estimation} on ViT-B/16. 
    \textbf{(a)}~Cross-term error $\eps^{(\mathrm{cross})}$. \textbf{(b)}~Correlation error $\eps^{(\mathrm{corr})}$. 
    \textbf{(c)}~Drift error $\eps^{(\mathrm{drift})}$. All three terms are relatively small across layers and tasks, indicating that $\mDelta_t^\top\mDelta_t$ is well-aligned with the final covariance $\mC_t^{(K)}$ (angular distance is normalized).}
    \label{fig:empirical-validation}
\end{figure}

In Theorem~\ref{thm:cov-estimation}, the \emph{cross-term error} $\eps^{(\mathrm{cross})}$ arises due to off-diagonal contributions from the double summation over iterations and double expectation over samples in~\eqref{eq:dtd-expansion}. 
In Figure~\ref{fig:ang-dist-cross}, we report the angular distance between $\overline{\mG}^\top\overline{\mG}$ and $\overline{\mS}$ across all datasets and transformer layers for ViT-B/16, consistently \cready{finding relatively small values for the cross-term contributions.}

The \emph{correlation error} $\eps^{(\mathrm{corr})}$ captures the coupling between per-sample activation outer products and output gradient norms, and this error term vanishes when these quantities are uncorrelated. 
In Figure~\ref{fig:ang-dist-corr}, 
we consistently find relatively small angular distances between $\overline{\mS}$ and $\widetilde{\mS}$ across datasets and transformer layers. Interestingly, a similar error term is encountered in KFAC~\citep{martens2015optimizing}, where activations and output gradients are assumed to be uncorrelated. 
In contrast, the correlation error in Theorem~\ref{thm:cov-estimation} vanishes as the correlation between activations and \emph{output gradient norms} approaches zero. We further analyze the correlation between these two quantities in Section~\ref{sec:exp--eigcov-estimate-analysis}.

\cready{
The \emph{drift error} $\eps^{(\mathrm{drift})}$ reflects how much the activation covariances change over the course of training, and is small when the covariances remain approximately stationary. In Figure~\ref{fig:ang-dist-drift}, we report the angular distance error and also find these errors to be comparatively smaller than the other two error terms. Altogether, these results suggest that in some settings $\mDelta_t^\top\mDelta_t$ is approximately proportional to the final covariance $\mC_t^{(K)}$ in accordance with Theorem~\ref{thm:cov-estimation}. 
}

\subsection{Analyzing the Impact of the \methodname Scaling Factor}\label{sec:method--scaling-coefficients}

The previous section shows that the covariance matrices can be recovered from parameter difference matrices, but only up to a scaling factor. That is, $\mC_t = \kappa_t \mDelta_t^\top\mDelta_t$ when all error terms vanish. In this section, we examine to what extent these scaling factors pose a problem in the context of model merging via the interference minimization merge rule. Recall that the minimizer of the interference objective has the closed-form solution
$
   \mW^\star = \textstyle\sum_t \mW_t \mC_t (\sum_{t^\prime} \mC_{t^\prime})^\dagger.
$
When all scaling factors are equal, using the \methodname estimates $\widehat{\mC}_t = \mDelta_t^\top\mDelta_t$ in place of the true covariances yields the same minimizer. 
This is due to the scale invariance of the minimizer (i.e., multiplying all the covariances by the same value will not affect the minimizer).
We formalize this statement in the following proposition (proof in Appendix~\ref{app:theorems--regmean-scale-invariance}).

\begin{restatable}[]{proposition}{regmeanScaleInvariance}
\label{prop:regmean-scale-invariance}
Let $\mC_t = \kappa_t \widehat{\mC}_t$ for $\kappa_t \in \mathbb{R}$, and define
the true and approximate minimizers as
$
\mW^\star = \textstyle\sum_t \mW_t \mC_t \left(\sum_{t'} \mC_{t'}\right)^\dagger
\text{ and }
\widehat{\mW} = \textstyle\sum_t \mW_t \widehat{\mC}_t ( \textstyle\sum_{t'} \widehat{\mC}_{t'})^\dagger.
$
If all $\kappa_t$ are equal, then $\mW^\star = \widehat{\mW}$.
\end{restatable}

Assuming $\mC_t = \kappa_t \mDelta_t^\top\mDelta_t$, then $\kappa_t = \|\mC_t\|_F / \|\mDelta_t^\top\mDelta_t\|_F$ (see Appendix~\ref{app:theorems--kappa-derivation} for details). In Figure~\ref{fig:kappa-ratios}, we plot the distribution of ratios $\kappa_i/\kappa_j$ over all dataset pairs $(i,j)\in\{1, \dots, T\}^2$, for each layer. We find that in many cases the distributions are concentrated around one, indicating that the scaling coefficients are approximately equal across tasks.

\begin{figure}
    \centering
    \includegraphics[width=1.0\linewidth]{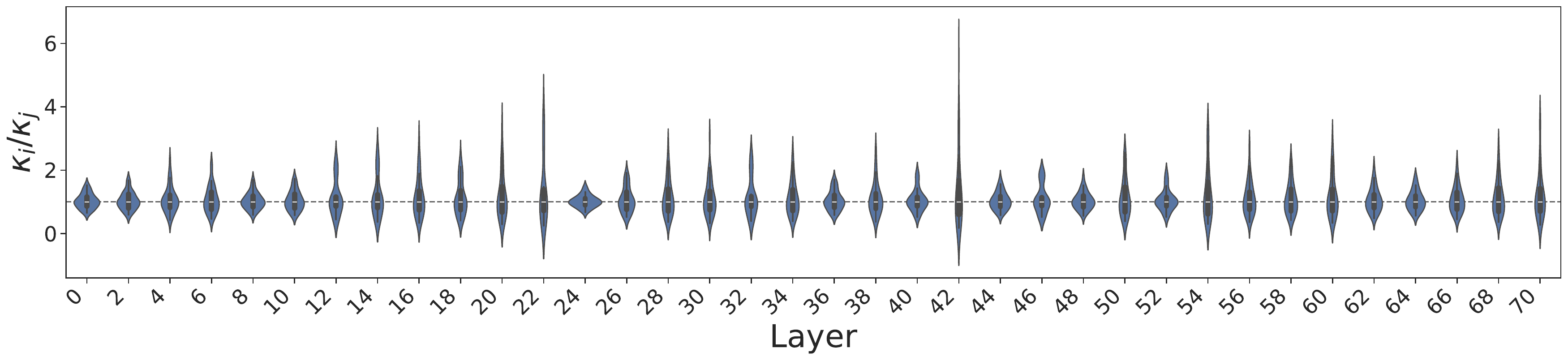}
    \caption{
        Distribution of scaling-coefficient ratios $\kappa_i / \kappa_j$ over all dataset pairs, for each layer in the ViT-B/16 model. $\kappa_i = \|\mC_i\|_F / \|\mDelta_i^\top\mDelta_i\|_F $ and measures the ratio of the norm of the covariance matrix to the norm of the \methodname estimate of the covariance matrix. 
    }
    \label{fig:kappa-ratios}
\end{figure}

\subsection{Theoretical Guarantees of Layer-wise Interference Minimization}
\label{sec:method--regmean-guarantees}

% NEW:
The layer-wise objective in~\eqref{eq:interference} is desirable as it admits a tractable closed-form solution (Lemma~\ref{lem:interference-minimization}). 
However, this comes at the cost of considering each layer independently, ignoring cross-layer interactions and nonlinearities. 
Thus, a natural concern is whether minimizing this layer-wise interference objective also minimizes the global objective.
Remarkably, \citet{sun2025cat} showed that layer-wise interference provides an upper bound on the \emph{negative transfer}, which is defined as
\begin{equation}
    % \label{eq:negative-transfer}
    \Delta\ell_t(\vx) = 
    |
    \ell \left( f\left(\vx;\, \vtheta_m \right) \right) - 
    \ell \left( f\left(\vx;\, \vtheta_t\right) \right)
    |,
\end{equation}
where $\ell:\R^D\to\R$ is a loss function applied to the output of network $f$ and $\vtheta_m$ is the merged parameter vector. In this section, we extend their result and show that negative transfer can be upper bounded by a summation of layer-wise interference errors and covariance estimation errors.

Suppose neural network $f$ decomposes layer-wise as $f = g^{(L)} \circ \cdots \circ g^{(1)}$, where $g^{(l)}(\vz;\,\vtheta^{(l)})$ is the function computed by layer $l$ on input $\vz$. Denoting the composed output at layer $l$ by $f^{(l)}$, we have that
$
f^{(0)}(\vx;\,\vtheta) = \vx, 
% \qquad
\;
f^{(l)}(\vx;\,\vtheta) = g^{(l)} 
(
f^{(l-1)}(\vx;\,\vtheta);\;\vtheta^{(l)}
).
$
Define the local error as
\begin{equation}\label{eq:local-interference}
\Delta g^{(l)}_t(\vz) = 
\|
g^{(l)}(\vz;\;\vtheta_m^{(l)}) -
g^{(l)}(\vz;\;\vtheta_t^{(l)})
\|.
\end{equation}
Intuitively, $\Delta g^{(l)}_t$ measures the change in the output of layer $l$ when only that layer's parameters are replaced by the merged parameters. 
% Let $\vtheta^{\star(l)}$ denote the minimizer of \eqref{eq:local-interference} and 
% $\Delta g^{\star(l)}_t(\vz)$ be the local error of the minimizer.
For linear layers, the minimizer $\vtheta^{\star(l)}$ of~\eqref{eq:local-interference} is given by the closed form solution in~\eqref{eq:interference-minimizer}.
We show in the following theorem that the expected negative transfer can be upper bounded by a summation of these local errors and covariance estimation errors. 

\begin{restatable}[]{theorem}{negativeTransferBound}
\label{thm:neg-transfer-bound}
Suppose that $g^{(l)}$ is $\gamma^{(l)}$-Lipschitz in its first argument and $\ell$ is $\beta$-Lipschitz. 
Let 
$\gS^\mathrm{(lin)} \subseteq \{1, \ldots, L\}$ 
denote the set of indices of linear layers. 
% V2:
For each linear layer $l \in \gS^{(\mathrm{lin})}$, let $\vtheta^{\star(l)}$ and $\vtheta_m^{(l)}$ denote the merged parameters obtained from the closed-form solution in~\eqref{eq:interference-minimizer} using true covariances $\mC_t^{(l)}$ and approximate covariances $\widehat{\mC}_t^{(l)}$
 \footnote{
With a slight abuse of notation, we use the superscript in $\mC_t^{(l)}$ to index layers, rather than to index training iterations as in Section~\ref{sec:method--covariance-estimation}.
}
, respectively. For all remaining layers, let $\vtheta_m^{(l)} = \vtheta^{\star(l)}$. Then,
\begin{multline}
\E\left[ \Delta\ell_t(\vx) \right]
\leq 
\sum_{l=1}^{L} \widetilde{\gamma}^{(l)} \,
\E\left[ \Delta g^{\star(l)}_t(\vz^{(l)}) \right] \\
\quad+
\sum_{l \in \gS^\mathrm{(lin)}}
\widetilde{\zeta}^{(l)}
\sum_t \| \mC_t^{(l)} - \widehat{\mC}_t^{(l)} \|_F 
\left(
\kappa_{\mW}^{(l)} \kappa_{\mS^\dagger}^{(l)} 
+ \kappa_{\mW}^{(l)}
\max\big\{ \| \mS^{(l)\dagger} \|_F^2, \, \| \widehat{\mS}^{(l)\dagger} \|_F^2 \big\}
\sum_{t'} \| \widehat{\mC}_{t'}^{(l)} \|_F
\right),
\end{multline}
where the expectation is taken over the $t$-th distribution $\train_t$,
$\vz^{(l)} = f^{(l-1)}(\vx;\,\vtheta_t)$,\; 
% $\Delta g_t^{\star(l)}(\vz) = \|(\mW^{\star(l)} - \mW_t^{(l)})\vz\|$,\; 
$\Delta g_t^{\star(l)}(\vz) =
\|
g^{(l)}(\vz;\;\vtheta^{\star(l)}) -
g^{(l)}(\vz;\;\vtheta_t^{(l)})
\|
$,\; 
$\mS^{(l)} = \sum_t \mC_t^{(l)}$,\; 
$\widehat{\mS}^{(l)} = \sum_t \widehat{\mC}_t^{(l)}$,\; 
$\widetilde{\gamma}^{(l)} = \beta \prod_{m=l+1}^{L} \gamma^{(m)} \text{ for } l < L$ and $\widetilde{\gamma}^{(L)} = \beta$, \; 
$\widetilde{\zeta}^{(l)} = \widetilde{\gamma}^{(l)}\, \E[\|\vz^{(l)}\|]$,\;
$\kappa_{\mW}^{(l)} = \max_t \|\mW_t^{(l)}\|_F$,
and $\kappa_{\mS^\dagger}^{(l)} = \|\mS^{(l)\dagger}\|_F$.
\end{restatable}
\begin{proof}
    The proof builds on Theorem 4.4 of \citet{sun2025cat}. See Appendix~\ref{app:theorems--negative-transfer-upper-bound} for details.
\end{proof}

Theorem~\ref{thm:neg-transfer-bound} makes explicit that the upper bound on the expected negative transfer consists of two distinct sources of error at each linear layer. The first term captures the interference error incurred when using the closed form minimizer in~\eqref{eq:interference-minimizer}, which arises due to the inability to exactly satisfy all tasks simultaneously. The second term captures the error introduced due to approximating covariance matrices and vanishes when $\widehat{\mC}_t^{(l)} = \mC_t^{(l)}$.

\section{Experiments}
% Two styles for organization of the experiments section i liked:
% 1. https://arxiv.org/pdf/2311.03099 (Super Mario) has dedicated setup
% 1b. https://arxiv.org/pdf/2505.06977 (better subsection names) Cat merging
% 2. https://arxiv.org/pdf/2502.04959 (Iso-C) setup mixed with each experiment

In this section we evaluate the performance of the \methodname merge rule on vision and language tasks (Section~\ref{sec:exp--eval-vision-lang}), as well as on reasoning tasks (Section~\ref{sec:exp--eval-reasoning}). 
We conclude by analyzing the covariance estimates of \methodname and its computational complexity (Section~\ref{sec:exp--eigcov-estimate-analysis}).

\paragraph{Datasets \& Models.} 
For vision tasks, we follow~\citet{ilharco2021openclip} and fine-tune ViT-B/16, ViT-B/32, and ViT-L/14 models on eight image classification datasets:
Cars~\citep{krause2013cars}, DTD~\citep{cimpoi2014dtd}, EuroSAT~\citep{helber2019eurosat}, GTSRB~\citep{stallkamp2011gtsrb}, MNIST~\citep{lecun1998mnist}, RESISC45~\citep{cheng2017resisc45}, SUN397~\citep{xiao2016sun397}, and SVHN~\citep{netzer2011svhn}. 
For language tasks, we fine-tune T5-Base and T5-Large~\citep{raffel2019exploring} on seven multiple-choice datasets: QASC~\citep{khot2020qasc}, WikiQA~\citep{cohen2018wikipassageqa}, QuaRTz~\citep{tafjord2019quartz}, PAWS~\citep{zhang2019paws}, Story Cloze~\citep{sharma2018tackling}, Winogrande~\citep{sakaguchi2020winogrande}, and WSC~\citep{wsc}. 
For reasoning tasks, we evaluate merging models trained via reinforcement learning with verifiable rewards, using OLMo-3-7B~\citep{olmo2025olmo} as the base model. Specifically, we merge three publicly available RL-Zero checkpoints trained on math, code, and instruction-following. 
Following the evaluation protocol in~\citet{olmo2025olmo}, we evaluate mathematical reasoning on AIME 2024 \& 2025. Meanwhile, coding and instruction-following abilities are evaluated using HumanEval~\citep{chen2021evaluating}, HumanEval+~\citep{liu2023your} and IFEval~\citep{zhou2023instruction}.

\paragraph{Baselines.} 
We compare against several merging methods including simple weight averaging, Task Arithmetic~\citep{ilharco2023editing}, RegMean~\citep{jin2023dataless}, WUDI~\cite{cheng2025whoever},
Iso-C~\citep{marczak2025notask}, TSV~\citep{gargiulo2025tsv}, and KnOTS~\citep{stoica2025knots}. Following~\citet{prateek2023ties}, we also report the individual \emph{Expert} performances, as well as the performance of the pretrained model which is referred to as the \emph{Zero-shot} performance \cready{(see Appendix~\ref{app:vision-language-experiments} for comparisons against additional baselines).}

\paragraph{Implementation Details.} 
Following~\cite{prateek2023ties}, for vision experiments we only fine-tune and merge the vision encoder. Meanwhile, for language experiments, embedding layers are always averaged. Following~\cite{jin2023dataless}, only 2D weight matrices are merged via each merging method's respective rule and all other parameters are averaged.  
In LoRA~\citep{hu2022lora} experiments, we follow the setup of~\cite{stoica2025knots} and use rank-16 adapters on all linear layers for vision models and language models. \cready{(see Appendix~\ref{app:implementation-details} for additional details)}.

Lastly, merging methods that rely on auxiliary data use the validation splits of datasets (i.e., RegMean to compute the empirical covariance matrices of the input activations and Task Arithmetic to do hyperparameter tuning to find the optimal scaling factor). 
In the data-free setting, Task Arithmetic uses a scaling factor of $\alpha=0.4$ following~\cite{prateek2023ties}, while Iso-C and TSV use $\alpha=1$, as both methods have been shown to be robust around this value.

\begin{figure}[tb]
\centering
\includegraphics[width=\linewidth]{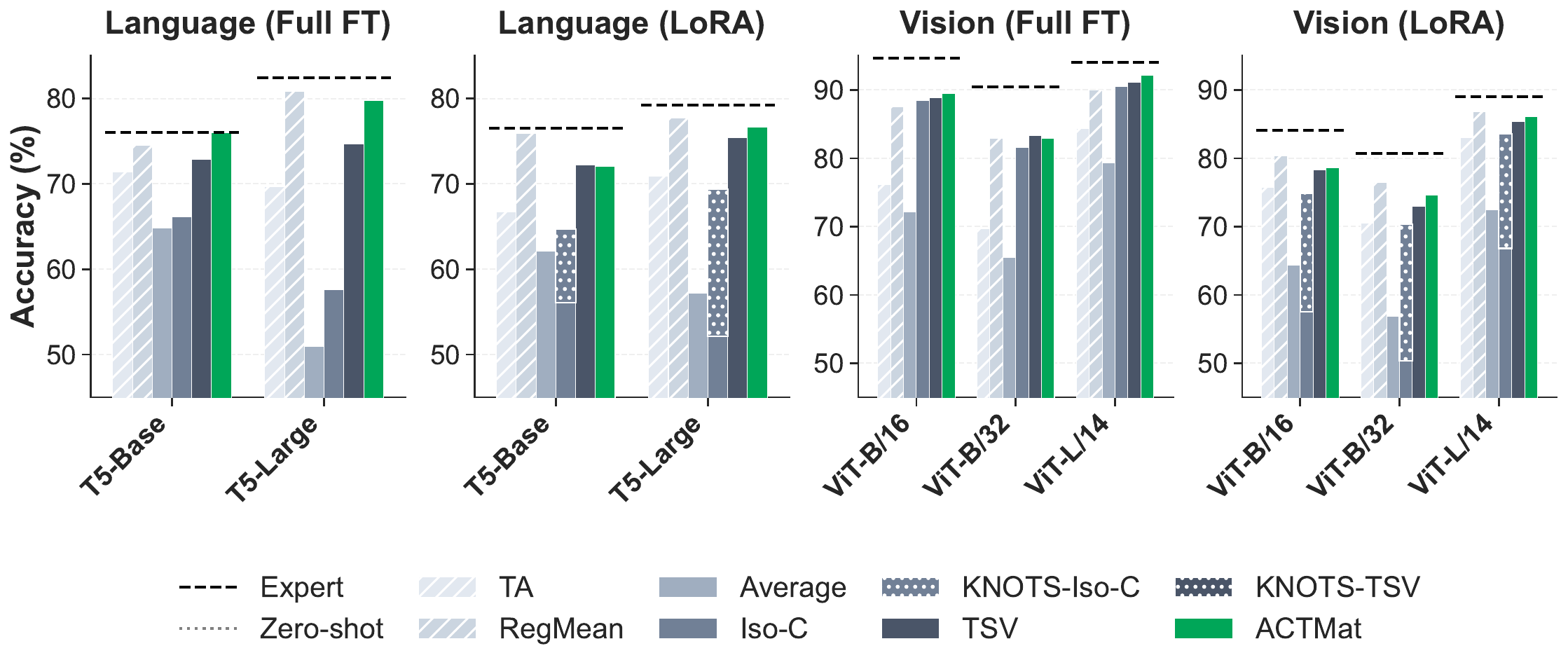}
\caption{
Comparison between test accuracy of merging methods across multiple settings (NLP models fine-tuned on 7 tasks and vision models fine-tuned on 8 tasks). 
\emph{Hatched bars} indicate that the method is not data-free. \emph{Stacked bars} with a dotted pattern indicate the performance of the method (bottom bar) and performance of the method when combined with KnOTS~\citep{stoica2025knots} for LoRA fine-tuned models~(improvement only for Iso-C).
}
\label{fig:performance--merging-standard-benchmark}
\end{figure}

\begin{table}
\centering
\begin{NiceTabular}{@{} l ccccccccc | c @{}}[colortbl-like]
\toprule
 \textbf{Method ($\downarrow$)}
 & \multicolumn{2}{c}{HumanEval}
 & \multicolumn{2}{c}{HumanEval+}
 & \multicolumn{2}{c}{AIME '24} 
 & \multicolumn{2}{c}{AIME '25} 
 & \multicolumn{1}{c}{IFEval}   
 & \multicolumn{1}{c}{Avg}    \\
\cmidrule(lr){2-3} \cmidrule(lr){4-5} \cmidrule(lr){6-7} \cmidrule(lr){8-9} \cmidrule(lr){10-10} \cmidrule(lr){11-11}
\textbf{Eval ($\rightarrow$)} & @1 & @10 & @1 & @10 & @1 & @32 & @1 & @32 & @1 & @1 \\
\midrule
\rowcolor{lightgray} Zero-shot & 50.8 & 82.2 & 47.3 & 78.6 & 22.1 & 66.7 & 21.6 & 53.3 & 30.7 & 34.5 \\
\rowcolor{lightgray} Expert & 62.0 & 85.3 & 57.0 & 84.7 & 38.2 & 76.7 & 30.7 & 66.7 & 82.4 & 54.1 \\
\midrule
RegMean & 59.1 & 85.7 & 54.8 & 81.1 & 30.8 & 73.3 & 27.8 & 60.0 & 74.7 & 49.4 \\
\midrule
Average & 59.9 & \textbf{85.9} & 54.5 & \underline{81.9} & 34.6 & \textbf{83.3} & 30.5 & 60.0 & 67.5 & 49.4 \\
Iso-C & 59.0 & \underline{84.4} & 53.8 & \textbf{82.9} & 32.5 & 76.7 & \underline{31.5} & 63.3 & 63.0 & 48.0 \\
TSV & \underline{62.6} & 82.5 & \textbf{58.4} & 80.9 & \underline{39.7} & \textbf{83.3} & \textbf{32.1} & \textbf{66.7} & \textbf{82.8} & \textbf{55.1} \\
WUDI & 62.0 & 80.2 & 56.3 & 76.6 & 38.9 & 70.0 & 29.7 & \textbf{66.7} & \underline{81.5} & 53.7 \\
ACTMat & \textbf{62.7} & 82.1 & \underline{56.9} & 75.9 & \textbf{39.9} & 80.0 & 29.8 & \textbf{66.7} & \underline{81.5} & \underline{54.2} \\
\bottomrule
\end{NiceTabular}
\caption{Comparison between merging methods in the RL Zero setting, where models are fine-tuned using RLVR starting from the OLMo-3-7B base model. We \textbf{bold} the best and \underline{underline} the second best data-free method for each dataset.}
\label{tab:performance--rlvr-results}
\end{table}

\subsection{Evaluation on Vision and Language Tasks}
\label{sec:exp--eval-vision-lang}

In Figure~\ref{fig:performance--merging-standard-benchmark}, we compare the performance of models merged using \methodname against baselines in both vision and language settings. 
We also report the results in tabular form in Appendix~\ref{app:vision-language-experiments}. 
In the full fine-tuning setting, \methodname achieves the highest average accuracy among all data-free methods on four out of five model configurations. The improvements are particularly pronounced on NLP tasks, where \methodname outperforms the next best data-free method, TSV, by \textbf{+3.1} percentage points on T5-Base (76.0 vs.\ 72.9) and \textbf{+5.1} percentage points on T5-Large (79.8 vs.\ 74.7). A similar trend is observed when LoRA fine-tuning is used, where \methodname also outperforms baselines on four out of five configurations.

We note that our RegMean results are substantially better than those consistently reported in many prior works~\citep{prateek2023ties, marczak2025notask, sun2025lot, guodong24neurips, yang2024adamerging}, though are in line with that of~\citet{tam2023merging, yu2024language}. 
We suspect this is due to subtle differences in implementations of RegMean for vision tasks~(see Appendix~\ref{app:implementation-details}).

Lastly, Figure~\ref{fig:accuracy-trajectory} reports the performance obtained by merging experts at intermediate fine-tuning checkpoints. Interestingly, in the ViT-B-16 experiment, \methodname outperforms other approaches by a larger gap earlier on in the fine-tuning process. Meanwhile, in the T5-Base experiments the performance gap between \methodname and other approaches remains large at the end of fine-tuning.

\subsection{Evaluation on Reasoning Tasks}
\label{sec:exp--eval-reasoning}
Beyond standard supervised fine-tuning, we investigate whether \methodname can effectively combine models trained via reinforcement learning. We merge three OLMo-3-7B experts trained using Reinforcement Learning from Verifiable Rewards (RLVR) on math, coding, and instruction following~\citep{olmo2025olmo}. 
The results shown in Table~\ref{tab:performance--rlvr-results} show that on average \methodname is the second best performing method, following TSV.

% \new{
% Notably, in  Table~\ref{tab:performance--rlvr-results} we find that RegMean's performance is significantly lower than other merging methods on the AIME'24 benchmark. In this setting, models are evaluated after producing long chain-of-thought reasoning traces (32K), meanwhile the RegMean statistics are only computed using the prompts. Thus this shift in distribution during evaluation can be one of the undelying cause
% }

\begin{figure}[t]
  \centering
  \begin{minipage}[t]{0.7\linewidth}
      \centering
      \includegraphics[width=\linewidth]{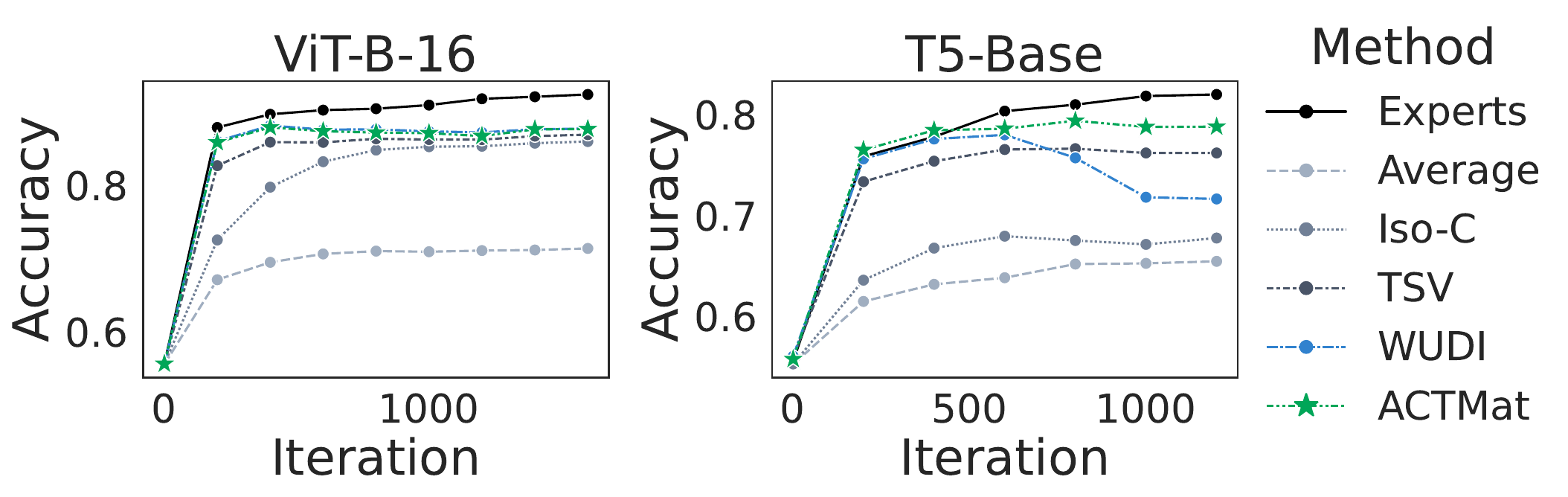}
      \captionof{figure}{\cready{Average accuracy of individual experts and merged models throughout fine-tuning.}}
      \label{fig:accuracy-trajectory}
  \end{minipage}
  \hfill
  \begin{minipage}[t]{0.27\linewidth}
      \centering
      \includegraphics[width=\linewidth]{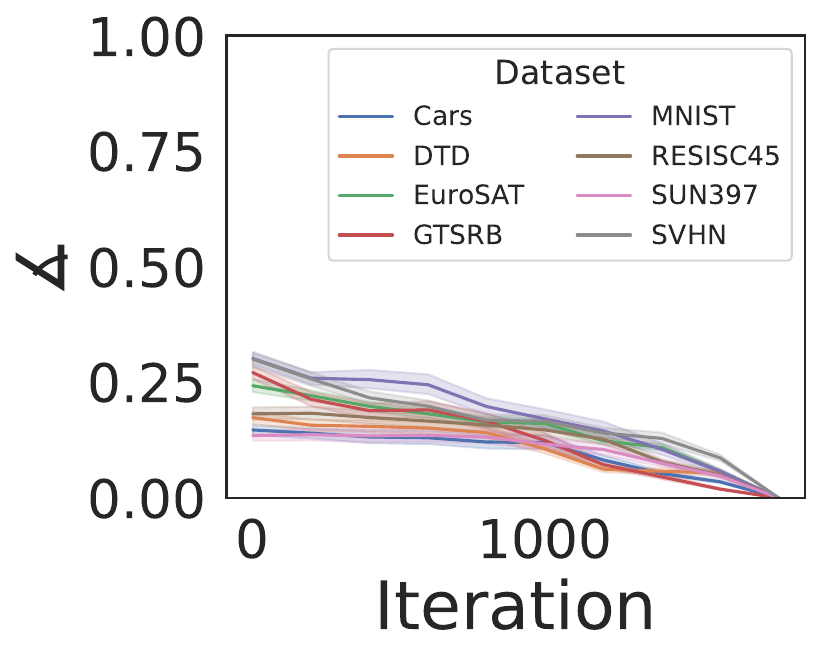}
      \captionof{figure}{Covariance error trajectory.}
      \label{fig:drift-trajectory}
  \end{minipage}
\end{figure}

\subsection{Analysis of the \methodname Estimator}
\label{sec:exp--eigcov-estimate-analysis}

\paragraph{Covariance stability throughout fine-tuning}
In Figure~\ref{fig:drift-trajectory} we analyze the angular distance between the final covariance matrices and the covariance matrices obtained at intermediate checkpoints. We find that the covariances matrices remain approximately stationary throughout fine-tuning. Notably, while this stationarity of the intermediate covariance matrices may lead to lower drift errors, strictly speaking it is not a necessary condition for attaining the low drift errors observed in Figure~\ref{fig:ang-dist-drift}.

\paragraph{Correlation between Activations and Gradients.}
As discussed in Section~\ref{sec:method--covariance-estimation}, underlying KFAC is the assumption that activations and activation gradients are uncorrelated~\citep{martens2015optimizing}. 
Analogously, if activations and activation gradient norms are uncorrelated, the correlation error term in Theorem~\ref{thm:cov-estimation} vanishes. We investigate to what extent activations and activation gradient norms are uncorrelated by plotting the Pearson correlation coefficients between these quantities in Figure~\ref{fig:correlation-coeffs}. Specifically, for each linear layer in the ViT-B/16 model, we compute the correlation coefficient between entries of matrix $\vz\vz^\top$ and activation gradient norms $\|\vg\|^2$.  We then plot the distribution of these values over entry indices and layer indices, for each dataset.

\paragraph{Covariance Estimation error of \methodname.}
The proposed merge rule~(\eqref{eq:eigencov}) leverages the  approximately proportional covariances estimates of \methodname. In Figure~\ref{fig:eigcov-estimation-error}, we analyze  the cosine similarity of the \methodname covariance estimates compared with empirical estimates (using 300 samples). We find that the \methodname covariance estimates are consistently more aligned with the empirical covariances, compared with identity approximations.

\paragraph{Complexity of Merging  Linear Layers via \methodname.}
In Table~\ref{tab:complexity-summary}, we compare the computational complexity of \methodname against other baselines when merging $T$ linear layers of size $N \times N$. Notably, \methodname's complexity matches that of both TSV and RegMean, though both TSV and RegMean incur additional costs not reflected in this table. Specifically, TSV relies on inherently sequential SVD operations, while RegMean requires a data-dependent pre-processing step. A more thorough analysis is provided in Appendix~\ref{app:compute-derivation}. Finally, we also report in Table~\ref{tab:complexity-summary} the wall-clock latencies associated with each method when merging the full T5-Large model on a single NVIDIA L40S GPU.

\begin{figure}[tb]
\begin{minipage}[tb]{0.3\linewidth}
\centering
\includegraphics[width=\linewidth]{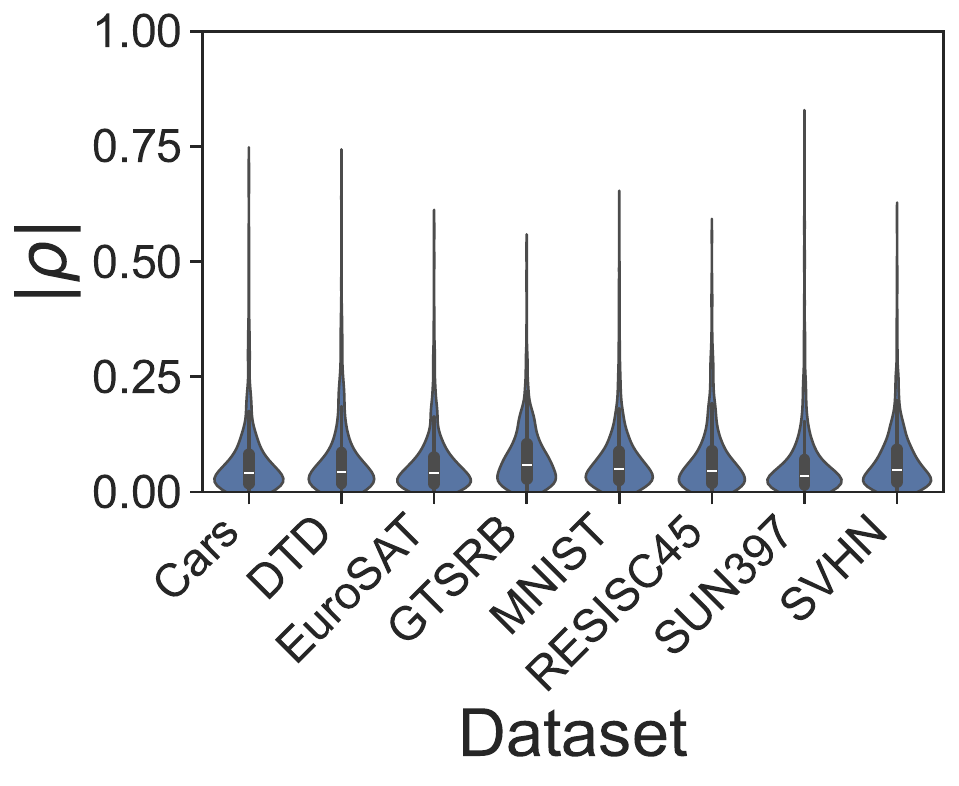}
\captionof{figure}{
Absolute Pearson correlation coefficients.
}
\label{fig:correlation-coeffs}
\end{minipage}\hfill
\begin{minipage}[tb]{0.3\linewidth}
\centering
    \includegraphics[width=\linewidth]{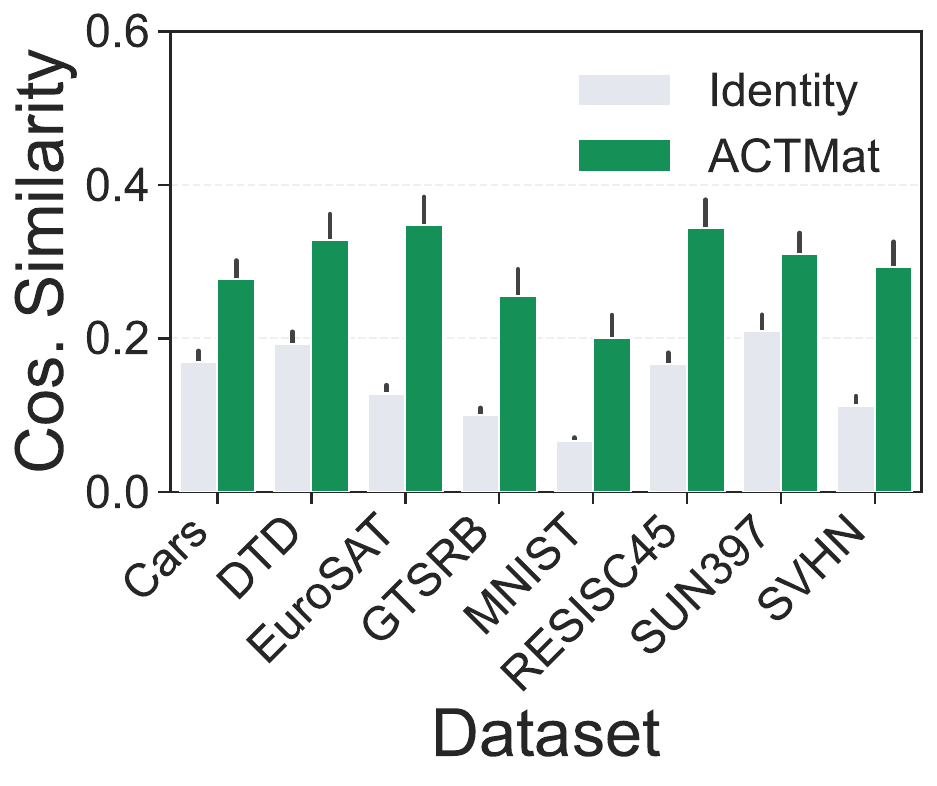}
    \captionof{figure}{
    Average layer-wise cosine similarity.
    }
    \label{fig:eigcov-estimation-error}
\end{minipage}\hfill
\begin{minipage}[tb]{0.34\linewidth}
\centering
\small
% 
% \rowcolors{2}{white}{gray!25}
\begin{tabular}{@{} lcc @{}}
\toprule
Method & FLOPs & Lat. \\
\toprule
Average & $\mathcal{O}(TN^2)$ & 0.05 \\
Iso-C & $\mathcal{O}(N^3 + TN^2)$ & 0.40 \\
TSV & $\mathcal{O}(TN^3)$ & 3.24 \\
RegMean & $\mathcal{O}(TN^3)$ & 4.66 \\
\methodname & $\mathcal{O}(TN^3)$ & 2.34 \\
\bottomrule
\end{tabular}
\captionof{table}{
Complexity of merging $T$ ($N \times N$) linear layers; \emph{Lat.}\ is wall-clock time in minutes.
% $L$ is the sample count for RegMean.
}
\label{tab:complexity-summary}
\end{minipage}
\end{figure}

\section{Limitations}
\cready{
As highlighted in Section~\ref{sec:method--covariance-estimation}, in the absence of errors the \methodname estimator is only proportional to the covariance matrices. 
Although our analysis in Section~\ref{sec:method--scaling-coefficients} showed that the scale ambiguity has a limited effect in our setting, this limitation can still affect the performance of the \methodname merge method in other settings.
}

% Although our analysis in Section~\ref{sec:method--scaling-coefficients} specified the conditions under which the scale ambiguity has a limited effect on the recovered minimizer, and showed that these conditions approximately hold for some layers, we expect that the scale ambiguity can still affect the performance of the \methodname merge method in general.

\section{Conclusion}

In this work, we presented \methodname, a principled approach to data-free model merging that combines covariance estimates derived from difference matrices with the interference minimization framework of RegMean. Our approach leverages three empirical findings that make covariance approximation from difference matrices possible. Namely, cross term error cancellations, uncorrelatedness of activations and activation gradients, and stationarity of covariances throughout fine-tuning. 

An interesting direction for future work is developing a theoretical understanding of why these three properties hold in practice, and characterizing the regimes 
under which each approximation is most or least accurate. 
More broadly, the availability of cheap, data-free covariance estimates may prove useful well beyond the merging setting, and we believe investigating such applications is a fruitful direction for future work.

\cready{
\section*{Acknowledgements}
We thank Stefan Horoi for helpful discussions and valuable feedback during the course of this work.
}

% \paragraph{Disclosure of LLM Usage}
% LLMs were used for the creation of TikZ figures, for assistance with code development, and  for assistance with rephrasing sentences for improved readability and flow.

\bibliography{colm2026_conference}
\bibliographystyle{colm2026_conference}

\newpage

\appendix

\section*{Appendix Table of Contents}
\startcontents
\printcontents{}{1}{\setcounter{tocdepth}{2}}

\section{Implementation Details}\label{app:implementation-details}

\subsection{Collecting covariances for RegMean.}
In PyTorch's \texttt{nn.MultiheadAttention}, the query, key, and value projections are stored as a single \texttt{nn.Parameter} (\texttt{in\_proj\_weight}) rather than as separate \texttt{nn.Linear} modules. The forward pass applies these projections via a direct call to \texttt{F.linear}. This means that using forward hooks registered on \texttt{nn.Linear} modules to collect covariance matrices, as done by~\cite{jin2023dataless}, will fail for models whose QKV projections are implemented using \texttt{nn.MultiheadAttention} such as the OpenCLIP model family~\citep{ilharco2021openclip}. In Table~\ref{tab:regmean-implmentation}, we report the performance of the vision models when using the \texttt{nn.MultiheadAttention} implementation compared against a custom implementation that circumvents the aforementioned issues. Notably, the custom implementation leads to significant performance gains, $\textbf{+4.5}$ percentage points on ViT-B/16, $\textbf{+3.6}$ on ViT-B/32 and $\textbf{+2.8}$ on ViT-L/14. We also perform the matrix inversion in double precision as lower precisions degraded performance in some cases.

\cready{
\subsection{Hyperparemeters for Vision \& Language Experiments}
In both vision and language experiments we obtain experts by fine-tuning the base models (ViT and T5) on their corresponding datasets. In Table~\ref{tab:sft-hparams} we report the hyper-parameters used for fine-tuning these models.
}

\begin{table}[tb]
\centering
\small
\begin{tabular}{lll}
\toprule
 & Vision (OpenCLIP ViT) & Language (T5) \\
\midrule
Optimizer         & AdamW & AdamW \\
Learning rate     & $1 \times 10^{-5}$ (cosine schedule w/ $500$-step warm-up) & $1 \times 10^{-4}$, constant \\
Weight decay      & $0.1$ & $0$ \\
Batch size        & $128$ & $256$ (base); $64$ (large) \\
Gradient clipping & $1.0$ & none \\
Precision         & \texttt{float32} & \texttt{bfloat16} \\
LoRA Rank         & $16$  & $16$ \\
\bottomrule
\end{tabular}
\caption{Fine-tuning configuration for the vision and language experts.}
\label{tab:sft-hparams}
\end{table}

\cready{
\subsection{Hyperparemeters for RL-Zero Experiments}
In these experiments we use the checkpoints released by \cite{olmo2025olmo}. Importantly, in these experiments as the experts are trained independently they do not all share the same chat template. Thus for each task in Table~\ref{tab:performance--rlvr-results}, we evaluate the performance of merged models using the corresponding chat template (e.g., chat template of the Math expert is used when performing the AIME evaluations). Finally, the evaluations are performed using OLMES,\footnote{\url{https://github.com/allenai/olmes}} as shown in Table~\ref{tab:olmo-task-ids}
}

\begin{table}[tb]
\centering
\small
\begin{tabular}{ll}
\toprule
Task & OLMES task id \\
\midrule
HumanEval  & \texttt{codex\_humaneval::tulu} \\
HumanEval+ & \texttt{codex\_humanevalplus::tulu} \\
AIME 2024  & \texttt{aime:zs\_cot\_r1::pass\_at\_32\_2024\_deepseek} \\
AIME 2025  & \texttt{aime:zs\_cot\_r1::pass\_at\_32\_2025\_deepseek} \\
IFEval     & \texttt{ifeval::tulu} \\
\bottomrule
\end{tabular}
\caption{OLMES task recipes used for evaluating the merged RL experts.}
\label{tab:olmo-task-ids}
\end{table}

\begin{table}[t]
\centering
\begin{tabular}{lccc}
\toprule
\textbf{Method} & ViT-B/16 & ViT-B/32 & ViT-L/14 \\
\midrule
RegMean (\texttt{nn.MultiheadAttention}) & 83.1 & 79.4 & 87.2 \\
RegMean           &  87.6 & 83.0 & 90.0 \\
\bottomrule
\end{tabular}
\caption{Performance comparison between RegMean implemented with default the \texttt{nn.MultiheadAttention} implementation vs a custom implementation enabling the covariances of query, key and value projection matrices to be captured.}
\label{tab:regmean-implmentation}
\end{table}

\section{Concurrent Work}

\methodname and ACE-Merging~\citep{xu2026acemerging} share the goal of estimating covariance matrices from task vectors, in order to enable data-free merging under the RegMean objective. However, there are a number of key distinctions between the two works, in their theoretical results, in how they empirically analyze the errors arising from the covariance estimation, and in the final methods they propose.

First, the main theorem of ACE-Merging relies on the strong assumption that each parameter matrix is trained independently using a squared error objective. The extent to which this assumption holds in practice remains unclear. In contrast, the ACTMat theorem does not require this independence assumption to hold.

Second, ACE-Merging’s theoretical result does not quantify the approximation error of the ACE-Merging estimator and assumes that any such errors vanish. In contrast, the theoretical result of ACTMat provides an interpretable decomposition of the covariance estimation error, making explicit that there are three sources of error. Namely, the errors arising from (i) cross terms, (ii) correlations between activations and gradients, and (iii) covariance drift. This theoretical result provides insights into which source contributes the most to the total error. Notably, the error decomposition provided by ACTMat can provide the basis for future work on studying conditions under which the most significant error terms (i) and (ii) can be reduced.

Finally, ACE-Merging was developed to handle heterogeneous tasks (i.e., task vectors obtained from models trained under differing conditions, such as varying numbers of epochs). This flexibility, however, comes at the cost of introducing several hyperparameters that may require careful tuning on new applications. (
$\tau, \kappa_\text{frac}, \epsilon$
). In contrast, ACTMat is hyperparameter-free and performs well across the evaluated tasks without any tuning.

\section{Theorems}

\subsection{Layer-wise Interference Minimization}
\label{app:theorems--interference-minimization}

\begin{lemma}
\label{lem:general_problem_interference-minimization}
Let $\mW_t \in\R^{D_o\times D_i}$ and $\mC_t \in \mathbb{R}^{D_i \times D_i}$ for $t = 1 \dots T$. Define $\mA := \sum_{t=1}^T  \mC_t \in \mathbb{R}^{D_i \times D_i}$ and $\mB:=\sum_{t=1}^T \mW_t \mC_t \in \mathbb{R}^{D_o \times D_i}$.
If the matrices $\mC_t$ are symmetric positive semidefinite (denoted from now on by $\mC_t \succeq 0$), then the matrix equation $\mW\mA = \mB$ always admits at least one solution, and the set of solutions is
\begin{equation}
\left\{ \mW^\star + \mZ \left( \mathbb{I}_{D_i} - \mA \mA^\dagger  \right) \mid  \mZ \in \mathbb{R}^{D_o \times D_i} \right\},
\end{equation}
with $\mW^\star = \mB \mA^\dagger$ being the minimum Frobenius norm solution.
Moreover, the solution is unique iff $\mA = \sum_{t=1}^T \mC_t$ is invertible, which is equivalent to $\ker\left(\sum_{t=1}^T \mC_t\right) = \bigcap_{t=1}^T \ker(\mC_t)=\{0\}$.
\end{lemma}

\begin{proof}
A solution to the equation $\mW\mA = \mB$ exists if and only if each row of $\mB$ belongs to the row space of $\mA$, i.e. $\mB=\mB\mA^\dagger\mA$.
We solve the equation row by row. Let $\vw \in \mathbb{R}^{D_i}$ be a row of $\mW \in\R^{D_o\times D_i}$ and $\vb \in \mathbb{R}^{D_i}$ the corresponding row of $\mB$. The equation $\mW\mA = \mB$ is equivalent to $\mA \vw = \vb$ for all $(\vw, \vb)$.
Thus, the existence of a solution is equivalent to $\vb \in \operatorname{Im}(\mA) = \ker(\mA)^\perp$ for all row $\vb$ of $\mB$, where we use the fact that $\operatorname{Im}(\mA) = \ker(\mA)^\perp$  since $\mA$ is symmetric. 
Therefore, $\vb \in \operatorname{Im}(\mA)$ if and only if $\vx^\top \vb = 0$ for all $\vx \in \ker(\mA)$. We now prove this condition, which also reads:  $\mB \vx = 0$ for all $\vx \in \ker(\mA)$.

Let $\vx \in \ker(\mA)$, so that $\mA\vx = 0$.
Then $0 = \vx^\top \mA \vx = \sum_{t=1}^T \vx^\top \mC_t \vx$, which is equivalent to $\vx^\top C_t \vx = 0 \ \forall t$ since $\vx^\top \mC_t \vx \ge 0 \forall t$ (by the positive semidefinitness of each $\mC_t$). 
Since each $\mC_t \succeq 0$, there exists $\mC_t^{1/2}$ such that $\mC_t = \mC_t^{1/2} \mC_t^{1/2}$. 
Thus
\begin{align}
\vx^\top \mC_t \vx = 0 \quad \forall t
& \Longleftrightarrow \|\mC_t^{1/2} \vx\|^2 = 0 \quad \forall t
\\& \Longleftrightarrow \mC_t^{1/2} \vx = 0 \quad \forall t
\\& \Longrightarrow \mC_t \vx = 0 \quad \forall t,
\end{align}
So
\begin{align}
\mB \vx 
= \sum_{t=1}^T  \mW_t \mC_t \vx
= 0.
\end{align}
%Thus, for each row $\vb$ of $\mB$,
%\begin{equation} \vb^\top \vx = 0 \quad \forall \vx \in \ker(\mA), \end{equation} which implies $\vb \in \operatorname{Im}(\mA)$. Therefore, the system $\mA\vw = \vb$ admits a solution for every row, and hence $\mW\mA = \mB$ admits a solution.

Therefore, the system $\mW\mA = \mB$ admits a solution.
It is easy to check that all solutions are of the form
\begin{equation}
\mW = \mW^\star + \mZ \left( \mathbb{I}_{D_i} - \mA \mA^\dagger  \right) \ \forall  \mZ \in \mathbb{R}^{D_o \times D_i}.
\end{equation}
So the solution is unique iff the free term always vanishes, that is, iff $\mA \mA^\dagger = \mathbb{I}_{D_i}$.
Since $\mA \mA^\dagger$ is the orthogonal projector onto $\operatorname{Im}(\mA)$, this happens iff $\mA$ is invertible, or equivalently, if $\mA$ is positive definite (since $\mA\succeq 0$ as a sum of positive semidefinite matrices), or equivalently, if $\{0\} = \ker\left(\sum_{t=1}^T \mC_t\right)=\bigcap_{t=1}^T \ker(\mC_t)$.
\end{proof}

\begin{lemma}
\label{lem:interference-minimization}
Let $\mW_t \in\R^{D_o\times D_i}$ for $t = 1 \dots T$. Define $\mC_t = \E_{\vz\sim\train_t}[\vz\vz^\top] \in \mathbb{R}^{D_i \times D_i}$ for each $t = 1 \dots T$, where $\vz$ denote a $D_i$-dimensional random vector distributed according to $\train_t$. Then the matrix $\mW^\star \in \mathbb{R}^{D_o \times D_i}$ defined by
\begin{equation}
\mW^\star = \left( \sum_{t=1}^T \mW_t \mC_t \right) \left( \sum_{t=1}^T \mC_{t}\right)^\dagger,
\end{equation}
is the minimum Frobenius norm solution to the problem
\begin{equation}
\label{eq:interference-restated}
\min_{\mW} g(\mW), \quad g(\mW) = \sum_{t=1}^T \E_{\vz\sim\train_t}\bigl[\|\mW\vz - \mW_t \vz\|_2^2\bigr].
\end{equation}
\end{lemma}
\begin{proof}
%Since the objective in~\eqref{eq:interference-restated} is convex, any stationary point is a global minimizer. Expanding and using the cyclic property of the trace,

Expanding $g(\mW)$ and using the cyclic property of the trace, we get
\begin{align}
g(\mW)
& = \sum_{t=1}^T \E_{\vz\sim\train_t} \left[ \vz^\top (\mW - \mW_t)^\top (\mW - \mW_t) \vz \right] 
\\& = \sum_{t=1}^T \E_{\vz\sim\train_t} \tr\left((\mW - \mW_t)^\top (\mW - \mW_t) \vz\vz^\top \right)
\\& = \sum_{t=1}^T \tr\left((\mW - \mW_t)^\top(\mW - \mW_t)\mC_t\right).
\end{align}
So
\begin{equation}
\nabla_{\mW} g(\mW) = \sum_{t=1}^T (\mW - \mW_t)(\mC_t + \mC_t^\top) =  2 \sum_{t=1}^T (\mW - \mW_t)\mC_t,
\end{equation}
where the last equality uses the fact that $\mC_t$ is symmetric.
%
%Setting the gradient to zero and rearranging the terms leads to
%\begin{align} \nabla_{\mW} g(\mW) = 0  \Longleftrightarrow \mW \mA = \mB, \end{align}
%
The function $g$ is convex and differentiable, and therefore its minimizers correspond to the solution of the equation $\nabla_{\mW} g(\mW) = 0$, which is equivalent to $\mW \mA = \mB$
with $\mA = \sum_{t=1}^T  \mC_t \in \mathbb{R}^{D_i \times D_i}$ and $\mB=\sum_{t=1}^T \mW_t \mC_t \in \mathbb{R}^{D_o \times D_i}$. 
By Lemma~\ref{lem:general_problem_interference-minimization}, this equation always admits a solution since each $\mC_t$ is a covariance matrix, and thus symmetric positive semidefinite.
The minimum-Frobenius-norm solution is
\begin{equation}
\mW = \mB \mA^\dagger = \left( \sum_{t=1}^T \mW_t \mC_t \right) \left( \sum_{t=1}^T \mC_{t}\right)^\dagger = \mW^\star
\end{equation}
\end{proof}

\subsection{Covariance Estimation}\label{app:theorems--covariance-estimation}
\covarianceEstimationThm*
\begin{proof}
After $K$ iterations of full-batch gradient descent with learning rate $\eta$, the difference matrix can be written as
\[
\mDelta_t 
= \mW^{(K+1)} - \mW^{(0)} 
= -\eta \sum_{k=0}^{K} \E_{\vx \sim \train_t}\!\left[\vg^{(k)}\vz^{(k)\top}\right] 
= -\eta\,\overline{\mG},
\]
and therefore $\mDelta_t^\top\mDelta_t = \eta^2\,\overline{\mG}^\top\overline{\mG}$. 
Successively applying the triangle inequality for angular distance,
\[
\measuredangle\!\big(\mDelta_t^\top\mDelta_t,\; \mC_t^{(K)}\big) 
\;\leq\; 
\measuredangle\!\big(\mDelta_t^\top\mDelta_t,\; \overline{\mG}^\top\overline{\mG}\big) 
\;+\; \measuredangle\!\big(\overline{\mG}^\top\overline{\mG},\; \overline{\mS}\big) 
\;+\; \measuredangle\!\big(\overline{\mS},\; \widetilde{\mS}\big) 
\;+\; \measuredangle\!\big(\widetilde{\mS},\; \mC_t^{(K)}\big).
\]
Since $\mDelta_t^\top\mDelta_t$ and $\overline{\mG}^\top\overline{\mG}$ are collinear, the first term vanishes, giving
\[
\measuredangle\!\big(\mDelta_t^\top\mDelta_t,\; \mC_t^{(K)}\big) 
\;\leq\; 
\eps^{(\mathrm{cross})} + \eps^{(\mathrm{corr})} + \eps^{(\mathrm{drift})}.
\]
When all three errors vanish, $\measuredangle(\mDelta_t^\top\mDelta_t,\; \mC_t^{(K)}) = 0$, which implies $\mDelta_t^\top\mDelta_t \propto \mC_t^{(K)}$.
\end{proof}

\subsection{Negative Transfer Upper Bound}
\label{app:theorems--negative-transfer-upper-bound}
% =======================
% NEW
% =======================
The layer-wise objective in~\eqref{eq:interference} is desirable as it admits a tractable closed-form solution. 
However, this comes at the cost of considering each layer independently, ignoring cross-layer interactions and nonlinearities. 
Thus, a natural concern is whether minimizing this layer-wise interference objective also minimizes some global objective that we ultimately care about.
Remarkably, \citet{sun2025cat} showed that layer-wise interference provides an upper bound on the \emph{negative transfer}, which is defined as
\begin{equation}
    \label{eq:negative-transfer}
    \Delta\ell_t(\vx) = 
    |
    \ell \left( f\left(\vx;\, \vtheta_m \right) \right) - 
    \ell \left( f\left(\vx;\, \vtheta_t\right) \right)
    |,
\end{equation}
where $\ell:\R^D\to\R$ is a loss function applied to the output of network $f$ and $\vtheta_m$ is the merged parameter vector. In this section, we extend their result and show that negative transfer can be upper-bounded by a summation of layer-wise interference errors and covariance estimation errors.

Suppose neural network $f$ decomposes layer-wise as $f = g^{(L)} \circ \cdots \circ g^{(1)}$, where $g^{(l)}(\vz;\,\vtheta^{(l)})$ is the function computed by layer $l$ on input $\vz$. Denoting the composed output at layer $l$ by $f^{(l)}$, we have that
$
f^{(0)}(\vx;\,\vtheta) = \vx, 
% \qquad
\;
f^{(l)}(\vx;\,\vtheta) = g^{(l)} 
(
f^{(l-1)}(\vx;\,\vtheta);\;\vtheta^{(l)}
).
$
Define the local error at layer $l$ as
\begin{align}
\label{eq:app-prop-error}
\Delta f^{(l)}_t(\vx) 
&=\|
f^{(l)}(\vx;\,\vtheta_m) - f^{(l)}(\vx;\,\vtheta_t)
\|, \\
\Delta g^{(l)}_t(\vz) &= 
\|
g^{(l)}(\vz;\;\vtheta_m^{(l)}) -
g^{(l)}(\vz;\;\vtheta_t^{(l)})
\|.
\end{align}
Intuitively, $\Delta g^{(l)}_t$ measures the change in the output of layer $l$ when only that layer's parameters are replaced by the merged parameters. 
We show in the following theorem that the expected negative transfer can be upper bounded by a summation of these local errors and covariance estimation errors.

\negativeTransferBound*
\begin{proof}
Let $\vz^{(l)} = f^{(l-1)}(\vx;\,\vtheta_t)$ and $\tilde{\vz}^{(l)} = f^{(l-1)}(\vx;\,\vtheta_m)$ denote the input to layer $l$ under the fine-tuned and merged parameters, respectively. By the $\beta$-Lipschitz continuity of $\ell$,
\begin{equation}
\label{eq:app-beta-lipschitz-expansion}
\Delta\ell_t(\vx)
= \big| \ell\big( f^{(L)}(\vx;\, \vtheta_m) \big) - 
\ell\big( f^{(L)}(\vx;\, \vtheta_t) \big) \big|
\leq \beta \, \Delta f_t^{(L)}(\vx), 
\end{equation}
where $\Delta f_t^{(L)}(\vx)$ is the propagated error as defined in~\eqref{eq:app-prop-error}. 
We can bound  $\Delta f_t^{(L)}(\vx)$ recursively. For each layer $l$, we have that
\begin{align}
\Delta f_t^{(l)}(\vx)
&= \big\| g^{(l)}(\tilde{\vz}^{(l)};\,\vtheta_m^{(l)})
- g^{(l)}(\vz^{(l)};\,\vtheta_t^{(l)})
\big\|
\notag \\
&= \big\|
g^{(l)}(\tilde{\vz}^{(l)};\,\vtheta_m^{(l)}) 
{\color{blue}
\,-\, g^{(l)}(\vz^{(l)};\,\vtheta_m^{(l)})
}
\notag \\
&\qquad
{\color{blue}
+\, g^{(l)}(\vz^{(l)};\,\vtheta_m^{(l)})
}
- g^{(l)}(\vz^{(l)};\,\vtheta_t^{(l)})
\big\| 
\label{eq:proof-add-sub} \\
&\leq \big\|
g^{(l)}(\tilde{\vz}^{(l)};\,\vtheta_m^{(l)}) 
- g^{(l)}(\vz^{(l)};\,\vtheta_m^{(l)})
\big\|
\notag \\
&\qquad
+ \big\|
g^{(l)}(\vz^{(l)};\,\vtheta_m^{(l)})
- g^{(l)}(\vz^{(l)};\,\vtheta_t^{(l)})
\big\| 
\label{eq:proof-triangle} \\
&\leq \gamma^{(l)} \, \big\| \tilde{\vz}^{(l)} - \vz^{(l)} \big\|
+ \Delta g_t^{(l)}(\vz^{(l)}) 
\label{eq:proof-lip-layer} \\
&= \gamma^{(l)} \, \Delta f_t^{(l-1)}(\vx)
+ \Delta g_t^{(l)}(\vz^{(l)}),
\label{eq:proof-recursion}
\end{align}
where~\eqref{eq:proof-add-sub} adds and subtracts $g^{(l)}(\vz^{(l)};\,\vtheta_m^{(l)})$,~\eqref{eq:proof-triangle} applies the triangle inequality and~\eqref{eq:proof-lip-layer} uses $\gamma^{(l)}$-Lipschitz continuity of $g^{(l)}$ in its first argument. Unrolling~\eqref{eq:proof-recursion} gives
\begin{align}
\Delta f_t^{(L)}(\vx)
\leq   \left( \prod_{l=1}^L \gamma^{(l)} \right) \Delta f_t^{(0)}(x) + \sum_{l=1}^L \left( \prod_{m=l+1}^L \gamma^{(m)} \right) \Delta g_t^{(l)}(\vz^{(l)}).
\end{align}
with the convention $\prod_{m=l+1}^{L} \gamma^{(m)} = 1$ for $l=L$.
Combining this with~\eqref{eq:app-beta-lipschitz-expansion} and using $\Delta f_t^{(0)}(\vx) = 0$ leads to
\begin{equation}
\label{eq:app-unrolled}
\Delta\ell_t(\vx) 
\leq \beta \sum_{l=1}^L \left( \prod_{m=l+1}^L \gamma^{(m)} \right) \Delta g_t^{(l)}(\vz^{(l)})  
= \sum_{l=1}^{L} \widetilde{\gamma}^{(l)} \,
\Delta g^{(l)}_t(\vz^{(l)}).
\end{equation}
with
\begin{equation}
\widetilde{\gamma}^{(l)} := 
\left\{
    \begin{array}{ll}
        \beta \prod_{m=l+1}^{L} \gamma^{(m)} & \mbox{for } l < L \\
        \beta & \mbox{for } l=L
    \end{array}
\right.
\end{equation}
Using the triangle inequality, we can bound $\Delta g^{(l)}_t(\vz^{(l)})$ as follows:
\begin{align}
\Delta g^{(l)}_t(\vz^{(l)}) 
&= \| g_t^{(l)}(\vz^{(l)}; \vtheta_m^{(l)}) - g_t^{(l)}(\vz^{(l)}; \vtheta_t^{(l)}) \| \notag\\ 
&= \| g_t^{(l)}(\vz^{(l)}; \vtheta_m^{(l)})
{\color{blue}
- g_t^{(l)}(\vz^{(l)}; \vtheta^{\star(l)})
+ g_t^{(l)}(\vz^{(l)}; \vtheta^{\star(l)})
}
- g_t^{(l)}(\vz^{(l)}; \vtheta_t^{(l)}) \| \notag\\
&\leq \| 
g_t^{(l)}(\vz^{(l)}; \vtheta_m^{(l)})
- g_t^{(l)}(\vz^{(l)}; \vtheta^{\star(l)}) 
\|
+ \|
g_t^{(l)}(\vz^{(l)}; \vtheta^{\star(l)})
- g_t^{(l)}(\vz^{(l)}; \vtheta_t^{(l)}) 
\| \notag\\
&= 
\Delta g_t^{\star(l)}(\vz^{(l)})
+ \|
g_t^{(l)}(\vz^{(l)}; \vtheta^{\star(l)})
- g_t^{(l)}(\vz^{(l)}; \vtheta_m^{(l)}) 
\| 
\label{eq:app-local-decomp}
\end{align}

Substituting~\eqref{eq:app-local-decomp} into~\eqref{eq:app-unrolled}, and noting that for $l\notin\gS^\mathrm{(lin)}$ the second term in \eqref{eq:app-local-decomp} is zero (since we let $\vtheta_m^{(l)} = \vtheta^{\star(l)}$ for such layers), leads to
\begin{align}
\Delta\ell_t(\vx) 
& \leq 
\sum_{l=1}^{L} \widetilde{\gamma}^{(l)} 
\Delta g^{\star(l)}_t(\vz^{(l)}) 
+
\sum_{l \in \gS^\mathrm{(lin)}}
\widetilde{\gamma}^{(l)} \,
\Big\| \mW_m^{(l)} \vz^{(l)} - \mW^{\star(l)} \vz^{(l)} \Big\|
\\& \leq 
\sum_{l=1}^{L} \widetilde{\gamma}^{(l)} \,
\Delta g^{\star(l)}_t(\vz^{(l)}) 
+
\sum_{l \in \gS^\mathrm{(lin)}}
\widetilde{\gamma}^{(l)} \,
\Big\| \mW_m^{(l)} - \mW^{\star(l)} \Big\|_F \,
\Big\| \vz^{(l)} \Big\|.
\end{align}

where $\mW^{\star(l)}$ and $\mW_m^{(l)}$ denote the weight matrices corresponding to parameter vectors $\vtheta^{\star(l)}$ and $\vtheta_m^{(l)}$, respectively. Taking expectations over $\vx \sim \train_t$,
\[
\E\left[ \Delta\ell_t(\vx) \right]
\leq 
\sum_{l=1}^{L} \widetilde{\gamma}^{(l)} \,
\E\left[ \Delta g^{\star(l)}_t(\vz^{(l)}) \right]
+
\sum_{l \in \gS^\mathrm{(lin)}}
\widetilde{\zeta}^{(l)} \,
\Big\| \mW_m^{(l)} - \mW^{\star(l)} \Big\|_F\,\,
\]
where $\widetilde{\zeta}^{(l)} = \widetilde{\gamma}^{(l)} \, \E[ \| \vz^{(l)} \| ]$. Finally,
by applying Lemma~\ref{lem:minimizer-difference} to the difference $\| \mW_m^{(l)} - \mW^{\star(l)} \|_F$ we get
\begin{multline}
\E\left[ \Delta\ell_t(\vx) \right]
\leq 
\sum_{l=1}^{L} \widetilde{\gamma}^{(l)} \,
\E\left[ \Delta g^{\star(l)}_t(\vz^{(l)}) \right] \\
\quad+
\sum_{l \in \gS^\mathrm{(lin)}}
\widetilde{\zeta}^{(l)}
\sum_t \| \mC_t^{(l)} - \widehat{\mC}_t^{(l)} \|_F 
\left(
\kappa_{\mW}^{(l)} \kappa_{\mS^\dagger}^{(l)} 
+ \kappa_{\mW}^{(l)}
\max\big\{ \| \mS^{(l)\dagger} \|_F^2, \, \| \widehat{\mS}^{(l)\dagger} \|_F^2 \big\}
\sum_{t'} \| \widehat{\mC}_{t'}^{(l)} \|_F
\right).
\end{multline}

\end{proof}

\begin{lemma}\label{lem:minimizer-difference}
Let $\mW^\star = \sum_t \mW_t \mC_t \mS^\dagger$ and $\mW_m = \sum_t \mW_t \widehat{\mC}_t \widehat{\mS}^\dagger$, where $\mS = \sum_t \mC_t$ and $\widehat{\mS} = \sum_t \widehat{\mC}_t$. Then,
\[
\Big\| \mW_m - \mW^\star \Big\|_F
\leq
\sum_t \| \mC_t - \widehat{\mC}_t \|_F
\Big(
\kappa_{\mW} \kappa_{\mS^\dagger}
+
\kappa_{\mW} 
\max\big\{ \| \mS^\dagger \|_F^2, \, \| \, \widehat{\mS}^\dagger \|_F^2 \big\}
\sum_{t'} \| \widehat{\mC}_{t'} \|_F
\Big),
\]
where $\kappa_\mW = \max_t\| \mW_t \|_F $ and $ \kappa_{\mS^\dagger} = \| \mS^\dagger \|_F$.
\end{lemma}

\begin{proof}
We have that
\begin{align}
\Big\| \mW_m - \mW^\star \Big\|_F
&=
\Big\|
\sum_t \mW_t \mC_t \mS^\dagger
- \mW_t \widehat{\mC}_t \widehat{\mS}^\dagger
\Big\|_F
\notag \\
&=
\Big\|
\sum_t \mW_t \mC_t \mS^\dagger
{\color{blue} - \mW_t \widehat{\mC}_t \mS^\dagger + \mW_t \widehat{\mC}_t \mS^\dagger }
- \mW_t \widehat{\mC}_t \widehat{\mS}^\dagger
\Big\|_F
\label{eq:lem-add-sub} \\
&=
\Big\|
\sum_t \mW_t (\mC_t - \widehat{\mC}_t) \mS^\dagger
+ \mW_t \widehat{\mC}_t (\mS^\dagger - \widehat{\mS}^\dagger)
\Big\|_F
\label{eq:lem-regroup} \\
&\leq \sum_t
\| \mW_t \|_F \, \| \mC_t - \widehat{\mC}_t \|_F \, \| \mS^\dagger \|_F
+ 
\sum_{t'}
\| \mW_{t'} \|_F \|\widehat{\mC}_{t'}\|_F \|\mS^\dagger - \widehat{\mS}^\dagger\|_F
\label{eq:lem-first-triangle-ineq} \\
&\leq \sum_t
\| \mW_t \|_F \, \| \mC_t - \widehat{\mC}_t \|_F \, \| \mS^\dagger \|_F
\notag \\
&\qquad
+ \sum_{t'} \| \mW_{t'} \|_F \, \| \widehat{\mC}_{t'} \|_F \,
\max\big\{ \| \mS^\dagger \|_F^2, \, \| \widehat{\mS}^\dagger \|_F^2 \big\} \,
\| \mS - \widehat{\mS} \|_F
\label{eq:lem-moore-perturbation} \\
&\leq \sum_t
\| \mW_t \|_F \, \| \mC_t - \widehat{\mC}_t \|_F \, \| \mS^\dagger \|_F
\notag \\
&\qquad
+ \sum_{t'} \| \mW_{t'} \|_F \, \| \widehat{\mC}_{t'} \|_F \,
\max\big\{ \| \mS^\dagger \|_F^2, \, \| \widehat{\mS}^\dagger \|_F^2 \big\} \,
\sum_{t''} \|\mC_{t''} - \widehat{\mC}_{t''} \|_F
\label{eq:lem-second-triangle-ineq} \\
&\leq
\sum_t \| \mC_t - \widehat{\mC}_t \|_F
\Big(
\kappa_{\mW} \kappa_{\mS^\dagger}
+
\kappa_{\mW} 
\max\big\{ \| \mS^\dagger \|_F^2, \, \| \, \widehat{\mS}^\dagger \|_F^2 \big\}
\sum_{t'} \| \widehat{\mC}_{t'} \|_F
\Big),
\label{eq:lem-final}
\end{align}
where~\eqref{eq:lem-add-sub} adds and subtracts $\sum_t \mW_t \widehat{\mC}_t \mS^\dagger$ and ~\eqref{eq:lem-moore-perturbation} applies Lemma~\ref{lem:pi-perturbation-bound}.
\end{proof}

\begin{lemma}[Pseudo-inverse Perturbation Bound, Theorem 2.1.~\citep{meng2010perturbation}]
\label{lem:pi-perturbation-bound}
Let $m$ and $n$ be two integers. Then,
\begin{equation}
\| \mA^\dagger  - \mB^\dagger  \|_F \le \max\{ \| \mA^\dagger  \|_F^2, \|  \mB^\dagger  \|_F^2  \} \|  \mA - \mB  \|_F \quad \forall \mA, \mB \in \mathbb{R}^{m \times n}.
\end{equation}
\end{lemma}

\subsection{Interference Objective Scale-Invariance}
\label{app:theorems--regmean-scale-invariance}

\regmeanScaleInvariance*
\begin{proof}
Since all $\kappa_t = \kappa$ for some scalar $\kappa$,  substituting $C_t = \kappa \hat{C}_t$ leads to
\[
\left(\sum_t W_t C_t\right)\left(\sum_{t'} C_{t'}\right)^{\dagger}
= \left(\kappa \sum_t W_t \hat{C}_t\right)\left(\kappa \sum_{t'} \hat{C}_{t'}\right)^{\dagger}
= \left(\sum_t W_t \hat{C}_t\right)\left(\sum_{t'} \hat{C}_{t'}\right)^{\dagger}.
\]
\end{proof}

\subsection{Kappa Derivation}\label{app:theorems--kappa-derivation}

Assuming that the covariance estimate satisfies $\mC_t = \kappa_t \widehat{\mC}_t$, then we have that $\|\mC_t\| = |\kappa_t| \|\widehat{\mC}_t\|$, and thus $ |\kappa_t| = \|\mC_t\| / \|\widehat{\mC}_t\|$. Furthermore, as the covariance matrix and covariance estimate  are both positive semi-definite matrices, we have that $\kappa_t  > 0$. Thus, $\kappa_t  = \|\mC_t\| / \|\widehat{\mC}_t\|$.

% \begin{table}[h]
% \centering
% \caption{Performance comparison across benchmarks with rank 4096. \textbf{Bold} indicates best, \underline{underline} indicates second best.}
% \label{tab:4096-comparison}
% \begin{tabular}{lcccc}
% \toprule
% \textbf{Benchmark} & \textbf{\methodname} & \textbf{Iso-C} & \textbf{Mean} & \textbf{TSV} \\
% \midrule
% BBH-CoT    & 0.685                          & 0.686                          & \underline{0.687}              & \textbf{0.690} \\
% HumanEval  & 0.621                          & \underline{0.711}              & \textbf{0.745}                 & 0.698 \\
% HumanEval+ & 0.548                          & \underline{0.693}              & \textbf{0.700}                 & 0.675 \\
% DROP       & \underline{0.633}              & 0.626                          & 0.625                          & \textbf{0.640} \\
% GSM8K      & \underline{0.677}              & 0.672                          & 0.672                          & \textbf{0.697} \\
% IFEval     & \underline{0.628}              & 0.575                          & 0.599                          & \textbf{0.634} \\
% MATH       & 0.248                          & \underline{0.250}              & 0.249                          & \textbf{0.251} \\
% PopQA      & 0.315                          & \underline{0.317}              & \underline{0.317}              & \textbf{0.320} \\
% \midrule
% Average    & 0.544                          & 0.566                          & \underline{0.574}              & \textbf{0.576} \\
% \bottomrule
% \end{tabular}
% \end{table}

\begin{table}[tb]
\centering
\small
% \begin{tabular}{lcccccc}
\begin{NiceTabular}{@{} lcccccc @{}}[colortbl-like]
\toprule
\textbf{Method ($\downarrow$)} & \textbf{Data-free} & \multicolumn{2}{c}{\textbf{NLP}} & \multicolumn{3}{c}{\textbf{Vision}} \\ 
\cmidrule(lr){3-4} \cmidrule(lr){5-7}
% Using a custom thickness or just a standard midrule here
\textbf{Model ($\rightarrow$)} & & \textbf{T5-B} & \textbf{T5-L} & \textbf{ViT-B/16} & \textbf{ViT-B/32} & \textbf{ViT-L/14} \\
\toprule
\rowcolor{lightgray} \textsc{Zero-shot} & - & 54.0 & 51.3 & 55.5 & 48.2 & 65.2 \\
\rowcolor{lightgray} \textsc{Experts} & - & 76.0 & 82.4 & 94.6 & 90.4 & 94.1 \\
\rowcolor{lightgray} \textsc{Multitask}	& - &  78.7 & 82.7 & 92.1 & 89.8 & 93.7 \\
\midrule
\textsc{RegMean} & \xmark & 74.5 & 80.8 & 87.6 & 83.0 & 90.0 \\
\textsc{TA} & \xmark & 71.4 & 69.7 & 76.1 & 69.7 & 84.3 \\
\midrule
\textsc{TA}  & \cmark & 60.5 & 59.8 & 24.3 & 26.5 & 41.8 \\
\textsc{Average} & \cmark & 64.8 & 50.9  & 72.2 & 65.5 & 79.3 \\
\textsc{TIES} & \cmark  & 72.3 & 71.7 & 79.5 & 72.8 & 85.2 \\
\textsc{DARE-TIES} & \cmark  & 73.1 & 73.7 & 75.6 & 67.7 & 84.5 \\
\textsc{Iso-C} & \cmark & 66.1 & 57.6 & 88.4 & 81.6 & 90.5 \\
\textsc{TSV} & \cmark & 72.9 & 74.7  & 88.8 & 83.3 & 91.1 \\
% \textsc{WUDI} & \cmark  & 74.0 & 77.3 & 89.4 & \underline{83.9} & 92.1 \\
% \textsc{ACE} & \cmark  & 70.1 & 74.4 & \textbf{90.7} & \textbf{85.2} & \textbf{92.2} \\
\textsc{WUDI} & \cmark  & 74.0 & 77.3 & 89.4 & \textbf{83.9} & 92.1 \\
% \textsc{ACE} & \cmark  & 70.1 & 74.4 & \textbf{90.7} & \textbf{85.2} & \textbf{92.2} \\
\textsc{\methodname} & \cmark & \textbf{76.0} & \textbf{79.8}  & \textbf{89.5} & 82.9 & \textbf{92.2} \\
\bottomrule
\end{NiceTabular}
\caption{
Comparison between merging methods across multiple settings (NLP models fine-tuned on 7 tasks and vision models fine-tuned on 8 tasks). All model parameters are fine-tuned. The \cmark symbol indicates that no data is used by the method. Results are reported on test sets.
} 
\label{tab:merging-results-full-ft}
\end{table}

\begin{table}[tb]
\centering
\small
\begin{NiceTabular}{@{} lcccccc @{}}[colortbl-like]
\toprule
\textbf{Method ($\downarrow$)} & \textbf{Data-free} & \multicolumn{2}{c}{\textbf{NLP}} & \multicolumn{3}{c}{\textbf{Vision}} \\ 
\cmidrule(lr){3-4} \cmidrule(lr){5-7}
\textbf{Model ($\rightarrow$)} & & \textbf{T5-B} & \textbf{T5-L}  & \textbf{ViT-B/16} & \textbf{ViT-B/32} & \textbf{ViT-L/14} \\
\midrule
\rowcolor{lightgray} \textsc{Zero-shot} & - & 54.0 & 51.3 & 55.5 & 48.2 & 65.2 \\
\rowcolor{lightgray} \textsc{Experts} & - & 76.5 & 79.2 & 84.1 & 80.7 & 89.0 \\
\midrule
\textsc{RegMean} & \xmark & 75.9 & 77.7 & 80.3 & 76.5 & 86.9 \\
\textsc{TA} & \xmark & 66.7 & 70.9 & 75.7 & 70.5 & 83.0 \\
\midrule
\textsc{TA} & \cmark & 54.2 & 55.0 & 66.1 & 53.5 & 77.6 \\
\textsc{Average} & \cmark & 62.2 & 57.2 &64.3 & 56.8 & 72.4 \\
\textsc{TIES} & \cmark & 67.3 & 68.7 & 75.1 & 69.0 & 82.7 \\
\textsc{DARE-TIES} & \cmark  & 66.5 & 69.8 & 75.9 & 70.2 & 83.5 \\
\textsc{Iso-C} & \cmark & 56.1 & 52.1 & 57.5 & 50.4 & 66.8 \\
\textsc{K-Iso-C} & \cmark & 64.7 & 69.4 & 74.9 & 70.4 & 83.6 \\
\textsc{K-TSV} & \cmark & 67.6 & 71.1 & 74.8 & 69.6 & 83.4 \\
\textsc{TSV} & \cmark & \textbf{72.2} & 75.4 & 78.4 & 73.0 & 85.4 \\
% \textsc{WUDI} & \cmark & 59.1 & 71.0 & 78.0 & 73.9 & 85.8 \\
\textsc{WUDI} & \cmark  & 59.1 & 71.0 & 78.0 & 73.9 & 85.8 \\
\textsc{\methodname} & \cmark & 72.0 & \textbf{76.7} & \textbf{78.6} & \textbf{74.6} & \textbf{86.1} \\
\bottomrule
\end{NiceTabular}
\caption{Comparison between merging methods across multiple settings (NLP models fine-tuned on 7 tasks and vision models fine-tuned on 8 tasks). Model parameters are fine-tuned using \textbf{LoRA}. The \cmark symbol indicates that no data is used by the method. Results are reported on test sets}
\label{tab:merging-results-lora}
\end{table}

% Table: Task-wise performance
\begin{table}[tb]
\centering
\footnotesize
\begin{NiceTabular}{@{} l cccccccc |c @{}}[colortbl-like]
\toprule
\textbf{Method} & Cars & DTD & EuroSAT & GTSRB & MNIST & RESISC45 & SUN397 & SVHN & Avg 
\\
\midrule
\rowcolor{lightgray} \textsc{Zero-shot} & 64.7 & 44.7 & 55.8 & 43.3 & 51.8 & 66.3 & 65.5 & 52.0 & 55.5 \\
\rowcolor{lightgray} \textsc{Experts} & 87.5 & 98.3 & 99.7 & 99.0 & 99.8 & 97.0 & 78.8 & 97.8 & 94.7 \\
\midrule
\textsc{RegMean} & 78.5 & 79.6 & 96.5 & 89.0 & 99.0 & 88.4 & 71.7 & 94.9 & 87.2 \\
\midrule
\textsc{Average} & 70.1 & 54.3 & 81.1 & 60.8 & 94.4 & 75.5 & 68.6 & 72.7 & 72.2 \\
\textsc{TA} & 61.5 & 53.8 & 76.7 & 72.5 & 98.3 & 65.6 & 47.8 & 87.1 & 70.4 \\
\textsc{TIES} & 73.2 & 69.0 & 88.3 & 73.4 & 98.7 & 79.2 & 65.9 & 88.1 & 79.5 \\
\textsc{Iso-C} & 82.7 & 85.1 & 96.3 & 91.8 & 98.8 & 91.6 & 75.0 & 86.1 & 88.4 \\
\textsc{TSV} & 81.8 & 85.8 & 96.8 & 93.8 & 99.3 & 89.3 & 71.6 & 92.4 & 88.8 \\
\textsc{WUDI} & 79.3 & 88.2 & 98.1 & 95.1 & 99.4 & 90.0 & 70.5 & 94.5 & 89.4 \\
% \textsc{ACE} & 81.7 & 92.8 & 97.5 & 93.3 & 99.4 & 90.8 & 74.9 & 95.1 & 90.7 \\
\textsc{ACTMat} & 79.8 & 91.4 & 97.0 & 92.3 & 99.5 & 89.4 & 72.9 & 93.8 & 89.5 \\
\bottomrule
\end{NiceTabular}
\caption{
Task-wise merging results
} 
\label{tab:performance-taskwise}
\end{table}

% Table: Scaling up to a larger number of experts
\begin{table}[tb]
\centering
\small
\begin{NiceTabular}{@{} lccccccccc @{}}[colortbl-like]
\toprule
\textbf{Method ($\downarrow$)} & \multicolumn{3}{c}{\textbf{ViT-B/16}} & \multicolumn{3}{c}{\textbf{ViT-B/32}} & \multicolumn{3}{c}{\textbf{ViT-L/14}} \\
\cmidrule(lr){2-4} \cmidrule(lr){5-7} \cmidrule(lr){8-10}
\textbf{Num. Experts ($\rightarrow$)} & \textbf{8} & \textbf{14} & \textbf{20} & \textbf{8} & \textbf{14} & \textbf{20} & \textbf{8} & \textbf{14} & \textbf{20} \\
\midrule
\rowcolor{lightgray} \textsc{Zero-shot} & 55.5 & 61.6 & 59.8 & 48.2 & 57.9 & 56.4 & 65.2 & 69.6 & 66.1 \\
\rowcolor{lightgray} \textsc{Experts} & 94.7 & 93.1 & 93.3 & 90.4 & 89.7 & 90.4 & 94.1 & 93.6 & 94.1 \\
\midrule
\textsc{RegMean} & 87.5 & 82.2 & 78.9 & 83.0 & 78.7 & 75.4 & 90.1 & 87.6 & 85.6 \\
\midrule
\textsc{Average} & 72.2 & 70.3 & 66.1 & 65.4 & 64.9 & 61.4 & 79.3 & 77.9 & 72.4 \\
\textsc{TIES} & 79.5 & 72.2 & 66.2 & 72.8 & 65.5 & 58.7 & 85.2 & 79.7 & 74.6 \\
\textsc{Iso-C} & 88.4 & 83.6 & 78.6 & 81.6 & 78.7 & 73.7 & 90.5 & 88.9 & 87.4 \\
\textsc{TSV} & 88.8 & 83.8 & 80.5 & 83.3 & \textbf{79.0} & \textbf{76.0} & 91.1 & 88.5 & 87.5 \\
\textsc{WUDI} & 89.4 & 80.8 & 66.0 & \textbf{83.9} & 75.6 & 57.5 & 92.1 & \textbf{89.0} & 83.1 \\
\textsc{ACTMat} & \textbf{89.5} & \textbf{83.9} & \textbf{80.7} & 82.9 & 78.0 & 75.6 & \textbf{92.2} & 88.6 & \textbf{87.7} \\
\bottomrule
\end{NiceTabular}
\caption{
Comparison between merging methods as the number of experts merged is varied.
}
\label{tab:performance-vs-experts}
\end{table}

\section{Vision and Language Experiments}
\label{app:vision-language-experiments}

\cready{
In this section we report in Tables~\ref{tab:merging-results-full-ft}~\&~\ref{tab:merging-results-lora} the results of Figure~\ref{fig:performance--merging-standard-benchmark}, in tabular form together with additional baselines including TIES~\cite{prateek2023ties} and DARE~\cite{yu2024language}. For the LoRA experiments in Table~\ref{tab:merging-results-lora} we report the performance of KnOTS~\cite{stoica2025knots} when combined with both TSV and Iso-C. Moreover, the task-wise performance breakdown is reported in Table~\ref{tab:performance-taskwise}. Lastly, we compare different merging methods when a larger number of experts is used in Table~\ref{tab:performance-vs-experts}.
}

% \section{Algorithm}

% \label{alg:actmat}

% \begin{algorithm}[H]
% \caption{ACTMat}
% \label{app:algorithm--actmat}
% \KwIn{Fine-tuned task weights $\{\mW_t\}_{t=1}^{T}$, pretrained weights $\mW_0$, with $\mW_t, \mW_0 \in \mathbb{R}^{m \times n}$}
% \KwOut{Merged weights $\mW^\star \in \mathbb{R}^{m \times n}$}
% \For{$t = 1, \dots, T$}{
%     $\mDelta_t \gets \mW_t - \mW_0$ \\
%     $\widehat{\mC}_t \gets \mDelta_t^\top \mDelta_t$ \tcp*{per-task covariance estimate}
% }
% $\mW^\star \gets \left(\sum_{t=1}^{T} \mW_t \widehat{\mC}_t\right) \left(\sum_{t=1}^{T} \widehat{\mC}_t\right)^\dagger$ \tcp*{RegMean minimizer}
% \Return $\mW^\star$
% \end{algorithm}

\section{Computational Complexity}
\label{app:compute-derivation}

\begin{table*}[tb]
\centering
\rowcolors{2}{white}{gray!25}
\begin{NiceTabular}{@{} l c c c @{}}[colortbl-like]
\CodeBefore
% Alternates from row 2: {row}{color1}{color2}
\rowcolors{2}{white}{gray!25}
\Body
\toprule
\textbf{Method} &                   \textbf{Merging FLOPs} & \textbf{Preprocessing FLOPs} & \textbf{\# of Exp. Op.} \\
\midrule
Average         &                                    $TN^2$ &                         - & 0\\
Task Arith.     &                               $(2T+1)N^2$ &                         - & 0 \\
RegMean         &              $(T+3)N^3 + (2T-2) N^2$  &                  $(2L-1)TN^2$ & 1 \\
\methodname         &               $(2T+3)N^3 + (3T-2) N^2$ &             -  & 1 \\
Iso-C         &              $ 23N^3 + (2T+2) N^2 + N$  &                -  & 1 \\
TSV         &   $(22T + 45) N^3 + (T+3)N^2$ & -       & T+2  \\
\bottomrule
\end{NiceTabular}
\caption{We report the computational cost of merging $T$ models for a single linear layer with equal input and output dimension $N$. For each method we show the merging FLOPs, the pre-processing FLOPs (applicable only to RegMean, where $L$ is the number of samples used to estimate the covariance matrices), and the number of expensive operations (\textbf{exp.\ op.}), i.e., SVDs or matrix inverses, which are inherently sequential).}
\label{tab:compute-detailed}
\end{table*}

We analyze the computational costs of merging $T$ models using different methods, focusing on linear layers since all methods compared default to simple averaging for other layer types. For simplicity, we assume equal input and output dimensions, denoted by $N$. We take the cost of a matrix multiplication to be $N^3$ FLOPs, a matrix inverse to be $2N^3$ FLOPs, and an SVD to be $20N^3$ FLOPs. For RegMean, $L$ denotes the number of samples used to estimate the covariance matrices.

Table~\ref{tab:compute-detailed} reports three quantities for each method including the merging FLOPs, the preprocessing FLOPs which apply only to RegMean, and the count of expensive operations. We report matrix inverses and SVDs separately from standard FLOPs because these operations are inherently sequential. In contrast, matrix multiplications and elementwise operations take full advantage of parallel computing hardware.
\textbf{Average:} Summing $T$ matrices requires $(T-1)N^2$ operations. 
Dividing by a scalar requires $N^2$ operations. The total is $TN^2$.

\textbf{Task Arithmetic:} Subtracting the pretrained checkpoint from each individual checkpoint requires $TN^2$ operations. 
Summing the resulting task vectors requires $(T-1)N^2$. 
Scaling the summed task vector by $\lambda$ requires $N^2$ operations and adding back the pretrained model requires another $N^2$ operations. 
The total is $(2T+1)N^2$ operations.

\textbf{RegMean:} Multiplying the Gram matrices by the weights requires $TN^3$ operations and summing the resulting matrices requries $(T-1)N^2$ operations. 
Summing the Gram matrices requires $(T-1)N^2$ operations and inverting the summed Gram matrix requires $2N^{3}$ operations. 
The final multiplication between the inverse of the sum of gram matrices and sum of gram-projected weights requires $N^3$ operations.
This results in $(T+3)N^3 + (2T-2) N^2$ operations. 

To compute the covariance statistics, an outer product of two vectors of dimension $N$ is computed, which requires $N^2$ FLOPs. This is done using $L$ samples, resulting in a total $LN^2$ FLOPs. Additionally, these outer products are added to a running sum, which requires $L-1$ summations for matrices of size $N\times N$. Thus, the total FLOP count is $(2L-1)N^2$ FLOPs for each model. Across $T$ models, this is $(2L-1)TN^2$ FLOPs.

\textbf{\methodname:} Requires the exact same operations as RegMean, except the additional operations to compute the task vectors and then multiply the task vector by its transpose to get the covariance matrix. Computing the task vectors requires $TN^2$ operations and multiplying the task vectors by its transpose requires $TN^3$ operations. Adding this to the operations for RegMean, this requires  $(2T+3)N^3 + (3T-2) N^2$ operations. 

\textbf{Iso-C:} Computing the task vectors requires $TN^2$ operations and summing them requires another $(T-1)N^2$. 
Computing the SVD on the summed matrix requires $22N^3$ operations and averaging the singular values requires $N$. 
Reconstructing the matrix from $UDV^T$ requires $N^2 + N^3$ operations. 
Finally, adding the reconstructed model to the pretrained model requires $N^2$ operations and scaling requires another $N^2$ operations. 
This results in $  23N^3 + (2T+2) N^2 + N$ operations. 

\textbf{TSV:} 
Computing the task vectors requires $TN^2$ operations. 
Computing the SVD for each task vector requires $T * 22N^3$ operations. 
Computing the SVD of the resulting $U$ and $V$ matrix requires $44N^3$ operations. 
Reconstructing the merged matrix requires $N^3 + N^2$ operations.
Finally, adding the reconstructed matrix to the pretrained weights requires $N^2$ operations and scaling requires another $N^2$ operations. 
This results in $(22T + 45) N^3 + (T+3)N^2$ operations.

\end{document}